\documentclass[10pt,twocolumn,twoside]{IEEEtran}

\usepackage{tcolorbox}
\usepackage{graphicx}
\usepackage{xcolor}
\usepackage{theorem}
\usepackage{times,amsmath,epsfig}
\usepackage{amssymb}
\usepackage{subfigure}
\usepackage{cite}
\usepackage{url}
\usepackage{multicol}
\usepackage{bm}
\usepackage{mathtools}

\usepackage{algorithmic}
\usepackage{algorithm}
\usepackage{tikz}
\usetikzlibrary{shapes,arrows}
\input{mysymbol.sty}

\newtheorem{mytheorem}{\bf Theorem}
\newtheorem{myexample}{\bf Example}

\newtheorem{myproposition}{\bf Proposition}
\newtheorem{remark}{\bf Remark}

\title{Online Change Point Detection for Weighted and Directed Random Dot Product Graphs}

\author{\IEEEauthorblockN{Bernardo Marenco, Paola Bermolen, Marcelo Fiori, Federico Larroca, and Gonzalo Mateos,~\IEEEmembership{Senior Member,~IEEE}}
\thanks{Work in this paper is supported in part by ANII (grant FMV\_3\_2018\_1\_148149) and the NSF (awards CCF-1750428, CCF-1934962 and ECCS-1809356). B. Marenco, P. Bermolen, M. Fiori and F. Larroca are with Facultad de Ingenier\'ia, Universidad de la Rep\'ublica, Uruguay. G. Mateos is with the Dept. of Electrical and Computer Eng., University of Rochester.  Emails: \{bmarenco,paola,mfiori,flarroca\}@fing.edu.uy and gmateosb@ece.rochester.edu. Part of the results in this paper were submitted to the \textit{2021 EUSIPCO} and \textit{Asilomar Conferences}~\cite{RDPG_EUSIPCO_21,RDPG_Asilomar_21}.}}

\begin{document}
\maketitle

\begin{abstract}%
Given a sequence of random (directed and weighted) graphs, we address the problem of online monitoring and detection of changes in the underlying data distribution. Our idea is to endow sequential change-point detection (CPD) techniques with a graph representation learning substrate based on the versatile Random Dot Product Graph (RDPG) model. We consider efficient, online updates of a judicious monitoring function, which quantifies the discrepancy between the streaming graph observations and the nominal RDPG. This reference distribution is inferred via spectral embeddings of the first few graphs in the sequence. We characterize the distribution of this running statistic to select thresholds that guarantee error-rate control, and under simplifying approximations we offer insights on the algorithm's detection resolution and delay. The end result is a lightweight online CPD algorithm, that is also explainable by virtue of the well-appreciated interpretability of RDPG embeddings. This is in stark contrast with most existing graph CPD approaches, which either rely on extensive computation, or they store and process the entire observed time series. An apparent limitation of the RDPG model is its suitability for undirected and unweighted graphs only, a gap we aim to close here to broaden the scope of the CPD framework. Unlike previous proposals, our non-parametric RDPG model for weighted graphs does not require a priori specification of the weights' distribution to perform inference and estimation. This network modeling contribution is of independent interest beyond CPD. We offer an open-source implementation of the novel online CPD algorithm for weighted and direct graphs, whose effectiveness and efficiency are demonstrated via (reproducible) synthetic and real network data experiments.
\end{abstract}

\begin{keywords}
Online change-point detection, graph representation learning, node embeddings, random dot product graphs.
\end{keywords}

\vspace{-0.5cm}
%
\section{Introduction}\label{sec:intro}

\IEEEPARstart{O}{nline} (or sequential) change-point detection (CPD) is the problem of deciding whether (and if so when) the generating process underlying an observed data stream has changed; see e.g.,~\cite{page1954continuous} for seminal work in the context of quality control. 
The goal is to flag a problem (in order to take corrective actions) as soon as it happens, while controlling the probability of false alarm. Unlike offline or batch processing (see e.g.,~\cite{truong2020selective}), in the online CPD setting we do not have access to the full data sequence which could well be infinitely long.

Given the ubiquity of datasets that are generated in a streaming fashion, online CPD is a timely research area with applications to sensor networks~\cite{he2018sequentialCPD}, financial markets~\cite{keshavarz2020gaussian}, or, social networks~\cite{peel2014evolvingCPD,kaushik2020signalsCPD}. As these examples suggest, data are increasingly high-dimensional and possibly non-Euclidean. Indeed, here we will consider \emph{network data streams} in the form of graph sequences. In a nutshell, given an incoming sequence of random (possibly weighted and directed) graphs, we want to signal if and when the data generating mechanism changes.\vspace{-0.3cm}

\subsection{Relation to prior work on online CPD for network data}\label{Ss:Prior_Work}

Sequential CPD approaches are often parametric, and follow the general premise of minimizing detection delay subject to a constraint on the test's type-I error. For network data existing methods look for changes in the graphs' distribution~\cite{he2018sequentialCPD,keshavarz2020gaussian,peel2014evolvingCPD}, their topology~\cite{kaushik2020signalsCPD} and community structure~\cite{zhang2020icassp}, or else the distribution of signals supported on the nodes~\cite{ferrari2019camsap}. Some of these~\cite{keshavarz2020gaussian,peel2014evolvingCPD,kaushik2020signalsCPD} are only applicable to undirected graphs. A sequential non-parametric, $k$-nearest neighbors-based approach was developed in~\cite{chen2019sequentialCPD}, solely requiring a pairwise distance between samples (e.g., the Frobenius distance between graph adjacency matrices). 
Unlike methods based on generative models, said distance is prone to overlooking simple changes in network structure; see the comparisons in Section \ref{ssec:synthetics}. A computationally-intensive model-based CPD effort advocates the Generalized Hierarchical Random Graph (GHRG) model in~\cite{peel2014evolvingCPD}, which monitors posterior Bayes factors for all partitions of the data over a sliding window. 
The approach in~\cite{wang2014dynamicCPD} is more general, as it considers the workhorse Stochastic Block Model (SBM). The distribution of two so-termed scan statistics is derived  to signal changes in the input graph sequence. 

Going beyond SBMs, the recent work~\cite{yu2021onlineCPD} considers an inhomogeneous Bernoulli graph; whereby the existence of an edge between a pair of nodes $(i,j)$ is a Bernoulli random variable with probability $P_{ij}$, independent of all other pairs. Each timestep, two statistics are computed for a logarithmic grid of previous instants to check whether they exceed a certain threshold. Evaluating these statistics necessitates computing the eigendecomposition of an $N\times N$ matrix ($N$ is the number of graph nodes). In addition to being computationally intensive, the algorithm in~\cite{yu2021onlineCPD} has to store all historical data in memory, which may pose a major hurdle even for moderate-sized networks. The procedure offers solid theoretical guarantees on the detection delay and average run length.

Here instead we resort to the Random Dot Product Graph (RDPG) model, a particular but very versatile case of the inhomogeneous Bernoulli graph~\cite{young2007random,priebe2018survey}. In RDPGs each node has an associated latent position in $\reals^d$ with $d\leq N$, and $P_{ij}$ is given by the inner product between the corresponding vectors. As we discuss in Section \ref{sec:rdpg}, RDPGs capture phenomena commonly encountered with real-world graphs (e.g., statistical dependencies among edges) and subsume the SBM as a special case, while still being amenable to analysis~\cite{priebe2018survey}. Moreover, RDPGs offer interpretability, an attractive feature that simplifies the explanation of the detected change-points.

\subsection{Paper outline and contributions}\label{Ss:Outline}

Building on~\cite{kirch2015estimating}, we assume a clean historical dataset with no change-points is available, from which we estimate the latent nodal vectors via the adjacency spectral embedding (ASE) in an offline training phase. As new data arrive in a streaming fashion during the operational phase, the novel online CPD algorithm (Section \ref{sec:online_cpd}) recursively updates a \emph{monitoring function} statistic whose null distribution we characterize analytically via asymptotic arguments. In addition to providing theoretical guarantees on the false alarm rate of the resulting online CPD scheme, an attractive feature is its limited memory footprint -- we store a single $N\times N$ matrix in memory (in addition to the estimated latent vectors, naturally). Moreover, the resulting lightweight statistic updates are an order-of-magnitude more efficient than those based on repeated eigendecompositions. Using simplifying approximations we derive conditions under which changes may go undetected.

An additional contribution is to extend the vanilla RDPG model~\cite{young2007random,padilla2019change} to accommodate weighted and directed graphs (digraphs), which we seamlessly adopt to perform online CPD for these general network models (Section \ref{sec:directed_weighted}). Extensions to digraphs are straightforward~\cite{priebe2017semiparametric}, but we carefully study those ambiguities inherent to the model (not discussed in previous work) which may challenge downstream CPD methods. 
Unlike previous RDPG proposals for weigthed graphs~\cite{tang2017robust,deford2016random},
our new non-parametric model in Section \ref{subsec:weighted_rdpg} does not require \emph{a priori} specification of the weights' distribution to perform provably consistent inference and estimation. We believe this contribution is significant in its own right, and beyond CPD it can e.g., impact node classification and visualization of network data.
Numerical tests in Section \ref{sec:simus} corroborate the effectiveness of the proposed online CPD method, using both simulated and real network datasets that we share in our Github repository. Concluding remarks and future directions are outlined in Section \ref{sec:conclusions}.

Relative to its conference precursors~\cite{RDPG_Asilomar_21,RDPG_EUSIPCO_21}, here we consider online CPD for weighted and directed graphs through a unified presentation along with full-blown technical details (including extended discussions, examples and unpublished proofs for all the theoretical results). Noteworthy novel pieces include: (i) examination of delay and change-detectability conditions; (ii) adoption of finite-memory (windowed) statistics; (iii) integrating the directed and weighted RDPG models for \emph{online} CPD; (iv) a consistency result for the weighted RDPG embeddings; and (v) a comprehensive and reproducible performance evaluation protocol. The latter offers comparisons with batch and online CPD baselines; an study of detection delay; the choice of monitoring function and thresholds; as well as applications to wireless and social networks.


\section{Preliminaries and Problem Statement}\label{sec:rdpg}

Here we introduce the necessary background on RDPG modeling and inference. The interested reader is referred to the comprehensive survey~\cite{priebe2018survey} for additional details about batch statistical network analysis. We then state the online CPD problem where the streaming graphs are modeled as RDPGs.

\subsection{Random dot product graphs}\label{subsec:vanilla_rdpg}

Consider an unweighted and undirected graph $G=(\ccalV,\ccalE)$, with nodes $\ccalV=\{1,\ldots,N\}$ and edges $\ccalE\subseteq \ccalV\times \ccalV$. If nodes $i$ and $j$ are connected in $G$, then the unordered pair $(i,j)\in \ccalE$. 
More general models involving directed and weighted graphs will be dealt with in Section \ref{sec:directed_weighted}. To start, we restrict ourselves to the simplest possible case for ease of exposition.  

In the RDPG model of $G$ each node $i\in \ccalV$ has an associated latent position vector $\bbx_i\in \mathcal{X} \subset \reals^d$, and edge $(i,j)$ exists with probability $P_{ij}=\bbx_i^\top\bbx_j$, independent of all other edges. We do not allow for self loops, hence $P_{ii}=0$ for all $i\in\ccalV$. The geometric interpretation is that nodes with large $\|\bbx_i\|_2$ tend to exhibit higher connectivity, whereas a small angle between $\bbx_i$ and $\bbx_j$ indicates higher ``affinity'' among $i$ and $j$. Note that the set $\mathcal{X}$ of possible $\bbx_i$ is such that $\bbx^\top\bby\in [0,1],\,\forall\, \bbx,\bby\in \mathcal{X}$. Just like with blockmodels and SBMs~\cite{kolaczyk_2017}, in general vectors $\bbx_i$ may be random, drawn from a (so-termed inner product) distribution in $\ccalX$. The dimensionality $d$ of the latent space is a model parameter, often much smaller than $N$.  

Thus, letting $\bbA\in\{0,1\}^{N\times N}$ be the random symmetric adjacency matrix of $G$ and $\bbX=[\bbx_1,\ldots,\bbx_N]^\top\in \reals^{N\times d}$ the matrix of latent vertex positions, the RDPG model specifies 
\begin{equation}\label{eq:rdpg_dist}
	\Pc{\bbA\given\bbX} = \prod_{i<j} (\bbx_i^\top\bbx_j)^{A_{ij}}(1-\bbx_i^\top\bbx_j)^{1-A_{ij}}.
\end{equation}
That is, \emph{given} $\bbX$, edges are conditionally independent with $A_{ij}\sim\textrm{Bernoulli}(\bbx_i^\top\bbx_j)$. 
\begin{myexample}\label{Ex:ER_SBM_Lantent}\normalfont
The RDPG model is a tractable yet expressive family of random graphs that subsume Erdös-R\'enyi (ER) and SBM ensembles as particular cases. Indeed, if $\bbx_i=\sqrt{p}\,\: \forall\, i$, we obtain an ER graph with edge probability $p$.    
An SBM with $M$ communities may be generated by restricting $\bbX$ to having only (at most) $M$ different columns (i.e.\ $|\mathcal{X}|=M$); see also~\cite{priebe2018survey} for additional examples. On the other hand, the RDPG is a particular case of the latent space model~\cite{hoff2002latent}, in which edge probabilities $P_{ij}=\kappa(\bbx_i,\bbx_j)$ are specified by means of a symmetric link function $\kappa$.
\end{myexample}

\subsection{Inference on RDPG via the adjacency spectral embedding}

Given the matrix $\bbX$ of latent vertex positions, the joint distribution in \eqref{eq:rdpg_dist} specifies the generative process to sample graphs from the RDPG model. We now discuss the associated inference (a.k.a. node embedding) problem. That is, how to estimate $\bbX$ having observed a graph stemming from an RDPG with adjacency matrix $\bbA$. 

In lieu of a maximum-likelihood estimator that is intractable beyond toy graphs~\cite{scheinerman2010rdpg}, the key intuition is that $\bbA$ is a noisy observation of 
\begin{gather}\label{eq:probamatrix}
	\bbP=\bbX \bbX^\top,
\end{gather}
the rank-$d$ matrix of edge probabilities $P_{ij}$, since $\E{\bbA\given \bbX}=\bbP$. It is thus natural to adopt the estimator
\begin{gather}\label{eq:ase_argmin}
	\hbX=\argmin_\bbX\|\bbA-\bbX\bbX^\top\|_F^2,\textrm{ s. to rank}(\bbX)=d.
\end{gather}
The solution to \eqref{eq:ase_argmin} is readily given by
\begin{gather}\label{eq:ase}
	\hbX = \hbV\hbLambda^{1/2},
\end{gather}
where $\bbA=\bbV\bbLambda\bbV^\top$ is the eigendecomposition of $\bbA$, $\hbLambda\in\reals^{d\times d}$ is a diagonal matrix with the $d$ largest eigenvalues of $\bbA$, and $\hbV\in\reals^{N\times d}$ are the corresponding $d$ dominant eigenvectors. We are assuming that $\hbLambda$ has only non-negative values, an apparent limitation that may be easily circumvented~\cite{rubindelanchy2017statistical}. The bias introduced by the implicit constraint $\textrm{diag}(\bbX\bbX^\top)\approx \mathbf{0}$ can be alleviated as well~\cite{scheinerman2010rdpg}. In practice, $d$ is likely unknown but can be estimated by looking for ``elbows'' on the so-termed eigenvalue scree plot~\cite{zhu2006automatic}. We find it is safer to overestimate $d$ (which will add some noise) than underestimate it, that will oversimplify the model and may e.g., hide change-points~\cite{RDPG_EUSIPCO_21}. Estimator \eqref{eq:ase} is known as the Adjacency Spectral Embedding (ASE), which is asymptotically normal and approaches $\bbX$ as $N\to\infty$ provided the true $d$ is chosen~\cite{priebe2018survey}. It is also possible to define an analogous normalized Laplacian spectral embedding for undirected $G$, which can be shown to enjoy similar desirable asymptotic properties to those of the ASE~\cite{priebe2018survey}.

Before moving on and stating the formal online CPD problem addressed in this paper, a couple of remarks on model identifiability and ASE variance reduction are in order.


\begin{remark}[Identifiability of latent positions]\label{rem:ambiguity}\normalfont
The RDPG model is identifable up to rotations of $\bbX$. To see this, consider an orthogonal matrix $\bbW\in\reals^{d\times d}$, and note that the rotated vectors $\bbX\bbW$ will produce the same probability matrix as in \eqref{eq:probamatrix}.   
Hence, the estimator \eqref{eq:ase} is unbiased up to an unknown rotation matrix $\bbW$, and the ambiguity should be accounted for when detecting changes on $G$'s distribution.
\end{remark}

\begin{remark}[ASE variance reduction]\label{rem:var_red}\normalfont Dispersion of ASE estimates can be reduced if one has access to multiple observations from the underlying RDPG. Indeed, let $\bbA[1]\ldots,\bbA[m]$ be an independent sequence of adjacency matrices, all adhering to an RDPG with latent position matrix $\bbX\in\reals^{N\times d}$. 
Define the mean adjacency matrix
\begin{gather}\label{eq:mean_adjacency}
	\displaystyle \barbA= \frac{1}{m}\sum_{\substack{t=1}}^{m} \bbA[t],
\end{gather}
and henceforth let $\hbX$ be the ASE decomposition of $\barbA$; i.e., the solution of \eqref{eq:ase_argmin} using $\barbA$ instead of $\bbA$. Since $\barbA$ is also an unbiased estimator of $\bbP$ and $\var{\bar{A}_{ij}}=\frac{1}{m}P_{ij}(1-P_{ij})$, then as $N\rightarrow\infty$ the estimated latent positions $\hbX$ will follow a normal distribution with variance scaled by $\frac{1}{m}$ relative to the variance of the ASE obtained from a single graph as in \eqref{eq:ase_argmin}~\cite{tang2018connectome}. The alternative of averaging individual ASEs is problematic due to the rotational ambiguity discussed in Remark \ref{rem:ambiguity}. Indeed, alignment of the (rotated) ASEs of a graph collection would entail solving several Procrustes distance minimization problems, or else computing the so-termed omnibus embedding~\cite{levin2017omnibus}.
\end{remark}

\subsection{Problem statement}

Suppose we acquire a batch of $m$ graphs as in Remark \ref{rem:var_red}, in which all matrices stem from the same RDPG model. We will refer to that sequence as the training data set, which is used in an \emph{offline} initialization phase to estimate model parameters from the null model. During the operational phase we observe a (possibly infinite) sequence of streaming adjacency matrices $\bbA[m+1], \bbA[m+2], \ldots $ , and 
would like to detect at what time $t>m$ (if any) the null model described in \eqref{eq:rdpg_dist} is no longer valid (i.e., drifts from the aforementioned RDPG model represent the alternative hypothesis). 
We tackle this CPD problem in an \emph{online} fashion, meaning graph observations $\{\bbA[m+k]\}_{k\geq 1}$ are sequentially and efficiently monitored as they are acquired, without having to store the whole multivariate time series. This way, the algorithm's computational complexity and memory footprint does not grow with $k$. Another attractive feature is the possibility of detecting the change in (pseudo) real-time, ideally soon after it occurs and with control on the probability of false alarm (i.e., type-I error).

We will also consider generalizations of the aforementioned baseline CPD problem in order to account for weighted and directed graph sequences. This calls for fundamentally re-examining the RDPG model to accommodate said observations -- especially in the weighted case --, as well as the associated embedding algorithms and the overall online CPD framework.

\section{Online Change-Point Detection}\label{sec:online_cpd}

Our idea to develop an online CPD framework for network data is to endow sequential CPD techniques with a graph representation learning substrate based on RDPGs.

\subsection{General algorithmic framework}\label{sec:general_framework}

We build on the so-called estimating function approach for sequential CPD~\cite{kirch2015estimating,kirch2018sequentialCPD}, which we markedly broaden to accommodate network data. The central notion behind this online CPD method is to consider a \textit{monitoring function} $\bbH$ of each streaming graph $\bbA[t]$, that should satisfy $\E{\bbH} = \bb0$ under the null hypothesis. If one monitors a cumulative sum of $\bbH$, that quantity should intuitively remain small provided there are no changes in the underlying model. If there is a change however, then $\E{\bbH}\neq \bb0$ and we should observe a drift in the trend of the sum.

As proposed in~\cite{kirch2015estimating} for a network-agnostic setting, we first estimate the parameters of the underlying null RDPG model using the training data set, i.e., we estimate the latent positions matrix $\bbX$. The estimation should be carried out with an \textit{estimating function} $\bbG$, where the estimated parameter $\hbX$ is the solution to a system of equations of the form
\begin{gather} \label{eq:suma_G}
	\sum_{t=1}^m \bbG(\bbA[t],\hbX) = \mathbf{0}.
\end{gather}

To define such a function for our problem, given the training data set we estimate $\bbX$ as the ASE corresponding to $\barbA$ [cf. \eqref{eq:mean_adjacency} and the discussion in Remark \ref{rem:var_red}]. Taking the derivative w.r.t. $\bbX$ of the objective function in \eqref{eq:ase_argmin} (with $\bbA \leftarrow \barbA$)  and setting it to zero, we arrive at
\begin{gather*}
	\sum_{t=1}^m \left(\hbX\hbX^\top-\bbA[t]\right)\hbX=  \mathbf{0},
\end{gather*}
suggesting the use of $\bbG(\bbA[t],\hbX)=\left(\hbX\hbX^\top-\bbA[t]\right)\hbX$ as the estimating function. 
Accordingly, $\bbG$ amounts to projecting the residual $\hbX\hbX^\top-\bbA[t]$ onto $\hbX$. 

In order to detect a change on the underlying model during the operational phase, we will track the cumulative sum (CUSUM) of a monitoring function $\bbH$ as new adjacency matrices arrive for $t\geq m+1$, namely
\begin{gather*}
	\bbS[m,k] = \sum_{t=m+1}^{m+k}\bbH(\bbA[t],\hbX).
\end{gather*}
While it is possible (and often natural) to use the same function for both estimation and monitoring (i.e. $\bbH = \bbG$), we show in Section \ref{ssec:synthetics} that adopting the residual itself instead of a projection yields in a more powerful detector. 
Thus, we choose
\begin{gather*}
	\bbH(\bbA[t],\hbX)=\hbX\hbX^\top-\bbA[t].
\end{gather*}
We reiterate here that the matrix $\hbX$ is computed during training, via the ASE of the average $\barbA$ of the adjacency matrices in the training set. Once monitoring starts, $\hbX$ is fixed and we do not recompute the ASE for new observations. 

Since all involved matrices are hollow and symmetric, we only need to consider entries, say, above the main diagonal. It will also prove useful in the analysis that follows to vectorize the resulting values. We thus define a vector function $\bbh$ as
\begin{gather}\label{eq:vec_triu_h}
	\bbh(\bbA[t],\hbX)=\textrm{vec}\left[\textrm{triu}\left(\hbX\hbX^\top-\bbA[t]\right)\right],
\end{gather}
where $\textrm{vec}(\textrm{triu}(\bbB))$ means arranging the entries above the main diagonal of matrix $\bbB$ in a vector.  
If $\bbB \in \reals^{N\times N}$, then $\textrm{vec}(\textrm{triu}(\bbB)) \in \reals^{r}$, with $r := \frac{N(N-1)}{2}$. 


If the norm of the partial sum
\begin{gather}\label{eq:cusum}
	\bbs[m,k] = \sum_{t=m+1}^{m+k}\bbh(\bbA[t],\hbX)
\end{gather}
exceeds a certain threshold, we will conclude that the model is no longer valid. Let us then denote our CUSUM statistic as
\begin{gather*}
	\Gamma[m,k] = \|\bbs[m,k]\|_2^2. 
\end{gather*}
In order to control the variance of $\Gamma[m,k]$  as $k$ grows, a weighting function $\omega[k]$ is also introduced. 
We use \mbox{$\omega[k]=(rk^{3/2})^{-1}$} and instead monitor $\omega[k]\Gamma[m,k]$; the reason for this choice is explained in the next section when we derive said variance for the null distribution.

All in all, the null hypothesis of no change will be rejected at the first time instant $k$ when
\begin{gather*}
	\omega[k]\Gamma[m,k] > c_{\alpha}[k],
\end{gather*}
where $c_{\alpha}[k]$ is a certain threshold that depends on the distribution of $\omega[k]\Gamma[m,k]$ under the null hypothesis and the prescribed type-I-error level $\alpha$. In the next section we will discuss how this threshold is chosen after characterizing the running statistic's null distribution. A pseudocode of the  online CPD method including the offline (training) and operational (monitoring) phases is tabulated under Algorithm \ref{A:online_CPD}. 

\begin{algorithm}[t!]
    \caption{Online change-point detection for RDPGs}
\label{A:online_CPD}
\algsetup{linenosize=\normalsize}
\begin{algorithmic}[1]
   \REQUIRE  Training graphs $\bbA[t], t=1\ldots m$.
   \STATE Compute the ASE $\hbX$ of $\barbA$ in \eqref{eq:mean_adjacency} (see Remark \ref{rem:var_red}) 
   \STATE Compute threshold function $c_{\alpha}$ (see Section \ref{ssec:imp_details})
   \STATE Initialize partial sum $\bbs[m,0]=\mathbf{0}$
   \FOR{$k = 1, 2, \dots$}
   \STATE Acquire graph $\bbA[m+k]$
   \STATE Compute monitoring function $\bbh(\bbA[m+k],\hbX)$
   \STATE Update CUSUM statistic $\Gamma[m,k]$ (see Remark \ref{rem:complexity}) 
   \IF{$w[k]\Gamma[m,k] > c_{\alpha}[k]$}
   \STATE Change point detected at time $k^*=k$
   \STATE \textbf{break}
   \ENDIF
   \ENDFOR 
   \RETURN $k^*$.
\end{algorithmic}
\end{algorithm}

\begin{remark}[Computational complexity]\label{rem:complexity}\normalfont
Efficient \emph{recursive} calculation of the cumulative monitoring function $\bbs[m,k]=\bbs[m,k-1]+\bbh(\bbA[m+k],\hbX)$ incurs $O(N^2)$ memory storage and computational complexity. The cost of forming the weighted CUSUM statistic $\omega[k]\Gamma[m,k]$ is of the same order. A single ASE is required in the \emph{offline} training phase to yield fixed edge probabilities estimates $\hbX\hbX^\top$. No embeddings have to be recomputed each time a new graph is observed. {To gain discriminative power in the statistical tests we perform, an alternative would be to examine the CUSUM statistic at every time point $t\in[m+1,\ldots,m+k]$. This comes at the price of increased computational complexity, since it would entail computing $k$ additional ASEs during the monitoring phase. This computational challenge is compounded with the need to derive the limiting distribution of the resulting stochastic process.}
\end{remark}

\subsection{Statistical analysis of the null distribution}\label{ssec:statistical_null}

In order to select the weighting and threshold functions, we will study the distribution of our statistic under the null hypothesis. We will first develop theory for the case when the ASE estimate is error-free, i.e., $\hbX\hbX^\top =\bbX\bbX^\top  =\bbP$. This way the estimated latent positions allow for a perfect reconstruction of the connection probability matrix. In practice, this will be valid when $m$ and/or $N$ are large enough. Since for some applications this may not be necessarily true, we will then extend the analysis for the imperfect estimation case.\vspace{2pt}


\noindent \textbf{Perfect ASE estimation.} In this favorable case one has\footnote{We have omitted the dependence of $\bbh$ on $t$ and $\hbX$ for clarity.} \mbox{$\bbh=\textrm{vec}\left[\textrm{triu}\left(\bbP-\bbA[t]\right)\right]$}, with $\E{\bbh} = \bb0$. The covariance matrix $ \bm\Sigma_H=\mathbb{E}(\bbh\bbh^\top)  \in\mathbb{R}^{r\times r}$ has null non-diagonal entries since the random variables $A_{ij}$ are mutually independent. The diagonal entries are $\var{A_{ij}}=P_{ij}(1-P_{ij})$.
In short,  $\bm\Sigma_H$ is a diagonal matrix whose nonzero entries are $p_l(1-p_l),\, l = 1,\dots,r$, with $p_l$ denoting the entries of $\textrm{vec}\left[\textrm{triu}\left(\bbP\right)\right]$ (i.e., a reindexing of $P_{ij}$).

Given this characterization of the first two moments of $\bbh$, the following proposition gives the asymptotic distribution of the CUSUM statistic $\Gamma[m,k]$ as $k\to \infty$. In practice, we rely on this limiting distribution as an approximation (for finite $k$) based on which we set the treshold $c_{\alpha}[k]$.

\begin{myproposition}\label{prop:perfect}
Suppose the perfect ASE estimation assumption $\hbX\hbX^\top =\bbX\bbX^\top  =\bbP$ holds. Then, as $k\to\infty$ the test statistic sequence converges in distribution, namely
	\begin{gather}\label{eq:gamma_dist}
		k^{-1}\Gamma[m,k] \stackrel{D}{\to}\sum\limits_{l=1}^{r} p_l(1-p_l) y_l^2,
	\end{gather}
	where $\{y_l\}_{l=1}^{r}$ are i.i.d. standard Gaussian random variables.
\end{myproposition}
\begin{proof}
Invoking the Central Limit Theoreom (CLT), as $k\to \infty$ the
	distribution of $k^{-1/2}\bbs[m,k]$ in \eqref{eq:cusum} converges to a multivariate Gaussian distribution $\mathcal{N}\left(\bb0,\bbSigma_H\right)$, i.e.,  $k^{-1/2}\bbs[m,k]\stackrel{D}{\to} (\bm\Sigma_H)^{1/2}\bby$, where $\bby$ is a standard Gaussian random vector. Hence, $k^{-1}\Gamma[m,k]=\|k^{-1/2}\bbs[m,k]\|_2^2$ also converges in distribution because
	\begin{align*}
		k^{-1}\Gamma[m,k]&{}=(k^{-1/2}\bbs[m,k])^\top k^{-1/2}\bbs[m,k]\\
		&{}\stackrel{D}{\to} (\bm\Sigma_H^{1/2}\bby)^\top\bm\Sigma_H^{1/2}\bby\\&{} =  \bby^\top \bm\Sigma_H \bby\\
		&{}= \sum_{l=1}^{r}p_l(1-p_l) y_l^2,
	\end{align*}
	which is the desired result in \eqref{eq:gamma_dist}.
\end{proof}
%

\begin{remark}[Convergence rate and network size]\normalfont
By bringing to bear Berry-Essen type results for the CLT, one can establish that the distribution of $k^{-1}\Gamma[m,k]$ converges to the limit stated in Proposition \ref{prop:perfect} at a rate $O(k^{-1/2})$, independent of $r$ and hence the graph size $N$; see e.g.,~\cite[Theorem 1.1]{Bentkus2003}. 
\end{remark}

Since $y_l\sim \ccalN(0,1)$ then $y_l^2 \sim \chi^2(1)$ (chi-squared distribution with one degree of freedom). By virtue of Proposition \ref{prop:perfect} and for sufficiently large $k$, we can approximate the mean and variance of $\Gamma[m,k]$ as
\begin{align}
	\E{\Gamma[m,k]}&{}\approx k\sum\limits_{l=1}^{r} p_l(1-p_l),\nonumber\\
	\var{ \Gamma[m,k]}&\approx  2k^2\sum\limits_{l=1}^{r} p_l^2(1-p_l)^2,\label{eq:mean_var_Gamma}
\end{align}
where we have used that the $\{y_l\}_{l=1}^{r}$ are mutually independent. 

To control the growing variance of $\Gamma[m,k]$, the weighting function for the perfect ASE case can be chosen as $\displaystyle \omega[k] = (rk)^{-1}$. The threshold $c_{\alpha}[k]$ is thus selected as the $(1-\alpha)$-quantile of the limiting distribution in \eqref{eq:gamma_dist}, which provides a type-I error of approximately $\alpha$. Next, we show that in the presence of estimation errors the weighting function will have to be readjusted accordingly.\vspace{2pt}

\noindent \textbf{Imperfect ASE estimation.} In this case, we will write
\begin{gather*}
	\hbX\hbX^\top-\bbA[t] = \bbX\bbX^\top-\bbA[t] + \hbX\hbX^\top- \bbX\bbX^\top,
\end{gather*}
where $\bbX$ is the true latent positions matrix (cf. $\bbP=\bbX\bbX^\top$). 
Defining the estimation error \mbox{$\bbE = \hbX\hbX^\top- \bbX\bbX^\top$}, then
\begin{gather}\label{eq:H_non_perfect}
	\bbh(\bbA[t],\hbX) =\textrm{vec}\left[\textrm{triu}\left(\bbX\bbX^\top-\bbA[t]\right)\right] + \bbe,
\end{gather}
where $\bbe = \textrm{vec}\left[\textrm{triu}\left(\bbE\right)\right]=[e_1,\ldots,e_r]^\top$. So the first term in \eqref{eq:H_non_perfect} corresponds to a perfect ASE, while the second one captures the estimation error stemming from an imperfect reconstruction of $\bbP$. Note that after training, $\bbe$ is fixed and it does not depend on $t$.

Using \eqref{eq:H_non_perfect} and by virtue of the CLT, it follows that for sufficiently large $k$ the distribution of $\bbs[m,k]$ can be well approximated by the multivariate Gaussian $\mathcal{N}\left(k\bbe,k\bm\Sigma_H\right)$. Standard calculations for the norm of a non-centered Gaussian vector suffice to assert that the distribution of $\Gamma[m,k]$ can be in turn approximated by the distribution of the random variable
\begin{gather}\label{eq:gamma_dist_non_perfect}
	\bar{\Gamma}=k\sum\limits_{l=1}^{r} p_l(1-p_l)\left( y_l + b_l\right)^2,
\end{gather}
where $\{y_l\}_{l=1}^{r}$ is an independent sequence of standard Gaussian random variables and $\{b_l\}_{l=1}^{r}$ are the entries of vector $\bbb = \sqrt{k} \bm\Sigma_H^{-1/2} \bbe$. Note that if the ASE estimation is perfect, then $\bbe = \bb0$ and we recover the distribution in Proposition \ref{prop:perfect}.

For large $k$, using \eqref{eq:gamma_dist_non_perfect} we can approximate the expected value and variance of $\Gamma[m,k]$ as in the error-free case. The difference here is that each summand $\left( y_l + b_l\right)^2\sim \chi^2(1,b_l^2)$, i.e., a non-central chi-squared distribution with one degree of freedom and parameter $\displaystyle b_l^2 = \frac{k}{p_l(1-p_l)}e_l^2$. 
Hence, one finds
\begin{align}
	\E{\Gamma[m,k]}&{}\approx k^2 \|\bbe\|_2^2 + k\sum\limits_{l=1}^{r} p_l(1-p_l)\nonumber \\ 
	&{}= k^2 \|\bbe\|_2^2 + k\|\bm\sigma\|_1,\label{eq:mean_Gamma_error}\\
	\var{\Gamma[m,k] }&{} \approx 4k^3 \sum\limits_{l=1}^{r} p_l(1-p_l)e_l^2 + 2k^2\sum\limits_{l=1}^{r} p_l^2(1-p_l)^2\nonumber\\ &{}= 4k^3 \bm\sigma^\top \bbe^2 + 2k^2\|\bm\sigma\|_2^2,\label{eq:var_Gamma_error}
\end{align}
where for notational convenience we defined the auxiliary vector $\bm\sigma$ with entries $\{p_l(1-p_l)\}_{l=1}^r$, and $\bbe^2$ denotes the entry-wise square of $\bbe$. The preceding arguments suffice to establish the following result on the convergence of $\Gamma[m,k]$.
\begin{myproposition}
In the general case, as $k\to\infty$ the test statistic sequence converges in distribution, namely
\begin{equation*}
\frac{\Gamma[m,k]-k^2 \|\bbe\|_2^2 - k\|\bm\sigma\|_1}{\sqrt{4k^3 \bm\sigma^\top \bbe^2 + 2k^2\|\bm\sigma\|_2^2}}\stackrel{D}{\to}y,    
\end{equation*}
where $y$ is a standard Gaussian random variable.
\end{myproposition}

Apparently, we need to choose $\omega[k] = (rk^{3/2})^{-1}$ to control the variance of the weighted statistic. This is because for large $k$, the term that dominates the variance expression \eqref{eq:var_Gamma_error} grows like $k^3$ [cf. $k^2$ in \eqref{eq:mean_var_Gamma}]. The detection threshold $c_{\alpha}[k]$ is thus set as the $(1-\alpha)$-quantile of the generalized chi-squared distribution defined in \eqref{eq:gamma_dist_non_perfect}, after weighting. We note that the resulting cumulative distribution function has a complex form which requires numerical integration to compute the desired quantiles; see also~\cite{cdf_approx_1,cdf_approx_2} for classic formulae to approximate said distribution function. As the next example shows, for particular cases the resulting distribution simplifies.
\begin{myexample}\label{Ex:ER}\normalfont
For an ER model with connection probability $p$ we have $p_l=p$ for all $l=1,\dots,r$ and \eqref{eq:gamma_dist_non_perfect} simplifies to
\begin{gather}\label{eq:ER_nonperfect}
\bar{\Gamma}_{\text{ER}} =
kp(1-p)u,\:\textrm{ with } u\sim \chi^2\left(r, \frac{k}{p(1-p)}\|\bbe\|_2^2\right).
\end{gather}
%
\end{myexample}

Alternatively, for threshold selection we will often resort to the mean plus three standard deviations
\begin{equation}\label{eq:threshold_simple}
    \text{th}[k] := \omega[k] \mathbb{E}_{bc}[k] + 3\sqrt{\omega^2[k] \text{var}_{bc}[k]}, 
\end{equation}
where $\mathbb{E}_{bc}$ is the expectation of the statistic before the change and $\text{var}_{bc}$ is its variance; given by~\eqref{eq:mean_Gamma_error} and \eqref{eq:var_Gamma_error}, respectively, using a suitable estimate of $\bbe$ described in Section \ref{ssec:imp_details} Numerical tests in Section \ref{ssec:synthetics} corroborate that this rule of thumb works well for all practical CPD purposes and it comes close to the true $0.99$-quantile. Moreover, having an analytic threshold expression facilitates studying the detection resolution of the online CPD procedure, the subject of the next section.

\subsection{Change detectability analysis}\label{subsec:power}

Let us examine what changes are detectable by the proposed online CPD algorithm, when using the simple thresholding rule $\textrm{th}[k]$ based on the derived mean and variance of the statistics's null distribution. To this end, we will assume that from a certain change-point $k=k_c$ onward, the sequence of graphs is generated by an RDPG with latent vectors $\bbY$ so that $\bbDelta:=\bbX\bbX^\top-\bbY\bbY^\top$ (i.e., the change is manifested through a perturbation on the resulting probability matrix). Given the expressive power of RDPGs~\cite{priebe2018survey}, the modeling assumption for $k\geq k_c$ comes with limited loss of generality. Henceforth, let $\bbdelta:=\textrm{vec}\left[\textrm{triu}\left(\bbDelta\right)\right]$.

If we are at a certain time $k>k_c$, the partial sum of the monitoring function is then (recall $\bbE = \hbX\hbX^\top- \bbX\bbX^\top$)
\begin{align*}
    \bbs[m,k] = {} &\sum_{t=m+1}^{m+k}\bbh\left(\bbA[t],\hbX\right) \\
    = {} &\sum_{t=m+1}^{m+k_c-1}\bbh\left(\bbA[t],\bbX\right)+\sum_{t=m+k_c}^{m+k}\bbh\left(\bbA[t],\bbY\right)\\
    {}&+k \bbe+(k-k_c)\bbdelta.
\end{align*}
Similar to the previous section, for large $k_c$ and $k$ we obtain a Gaussian vector with independent entries; mean $k\bbe + (k-k_c)\bbdelta$ and covariance matrix $k_c\diag[\bm\sigma_X] + (k - k_c)\diag[\bm\sigma_Y]$, where $\bm\sigma_X$ and $\bm\sigma_Y$ are the auxiliary vectors defined in \eqref{eq:var_Gamma_error} corresponding to $\bbX$ and $\bbY$, respectively. This results in a CUSUM statistic with mean approximately equal to 
\begin{align}
    \E{\Gamma[m,k]}\approx {}& \left\|k\bbe+(k-k_c)\bbdelta\right\|^2_2 \nonumber\\
    {}&+ k_c\|{\bm\sigma_X}\|_1 + (k-k_c)\|{\bm\sigma_Y}\|_1.
    \label{eq:mean_after_change}
\end{align}
In the long run as $k\to\infty$, the dominant term will be the first one, which when weighted by $\omega[k] = (rk^{3/2})^{-1}$ amounts to $\omega[k]\E{\Gamma[m,k]}\approx k^{1/2}\left\|\bbe+\bbdelta\right\|^2_2/r$. Given that $\omega[k]\Gamma[m,k]$ has finite variance and that on this asymptotic regime $\textrm{th}[k]\approx k^{1/2}\|\bbe\|^2_2/r$ plus a constant, we have established that changes are detectable as long as
\begin{gather}\label{eq:power}
    \|\bbe+\bbdelta\|^2_2> \|\bbe\|_2^2 \:\Rightarrow 2\|\bbe\|_2\cos{\theta}+\|\bbdelta\|_2>0,
\end{gather}
where $\theta$ is the angle between $\bbe$ and $\bbdelta$. It thus follows that a large value of $\|\bbdelta\|_2$ aids detectability, as expected. The same happens for small values of the estimation error magnitude $\|\bbe\|_2$, and in the idealized perfect estimation scenario we find all changes will be detected in the long run. Naturally, condition  \eqref{eq:power} is sufficient for changes to be detected, but not necessary. On the imperfect scenario, the resulting model estimation error will result in small changes likely going undetected provided $\theta\in(\frac{\pi}{2},\frac{3\pi}{2})$. On top of this angular requirement, a change may be missed when the ``perturbation-to-imperfection'' ratio is small, i.e., $\frac{\|\bbdelta\|_2}{\|\bbe\|_2}<2|\cos{\theta}|$.  

The following simple example offers additional insights on the feasibility of the condition \eqref{eq:power}.
\begin{myexample}\label{Ex:ER_detectability}\normalfont
Consider a sequence of ER graphs with connection probability $p$, which at a certain time-step $k_c$ changes to $q=p-\Delta$. In Appendix \ref{app:bound} we show that the following bound 
\begin{gather*}
    \hspace{-4cm}\Pc{\|\bbe+\bbdelta\|^2_2> \|\bbe\|_2^2}\geq \\1 - \frac{8(1-p)}{\Delta^2N^2(N-1)m}\left[\frac{1-p}{Nm}+2(N-1)p\right]
\end{gather*}
on the probability of satisifying the detectability condition \eqref{eq:power}
holds asymptotically in $N$. This means that if $\Delta^2N^2m$ goes to infinity as $N$ grows, then the change will be detected with high probability. In other words, the method detects changes $\Delta$ up to an order of $N^{-1}m^{-1/2}$. This example further illustrates that Algorithm \ref{A:online_CPD}'s performance improves with growing $m$ (the size of the training set) as well as $N$ (the number of nodes).
\end{myexample}

\subsection{Further implementation details}\label{ssec:imp_details}

We close this section with some necessary implementation details for Algorithm \ref{A:online_CPD}. These pertain to the calculation of the threshold and the possibility of utilizing windowed statistics as alternatives to the the cumulative sum \eqref{eq:cusum}. \vspace{2pt}

\noindent \textbf{Threshold calculation.} The procedure outlined in Section \ref{ssec:statistical_null} requires prior knowledge on the values of $\bbP$ and $\bbe$ in order to set the threshold $c_{\alpha}[k]$. This will be the case if one uses the exact $(1-\alpha)$-quantile of the null distribution, approximate formulae, or, simply $\text{th}[k]$ in \eqref{eq:threshold_simple}. In most applications the values of $\bbP$ and $\bbe$ are unknown, so it is necessary to estimate them from the observations in the training set. 

For $\bbP$ we simply use the plugin estimator $\hbP = \hbX\hbX^\top$, i.e., we estimate $\bbP$ using the ASE of $\barbA$ in \eqref{eq:mean_adjacency}, computed over the training set. 
Characterization of the statistical properties of $\bbE$ (and subsequently $\bbe$) is challenging in general. Even for the simple ER model, the study of $\bbE$ is non-trivial as shown in Appendix \ref{app:bound}. Therefore, we opted for a data-driven approach to form point estimates of $\bbE$ by performing ``leave-one-out'' passes over the training set: we randomly select an index $j$ in $1,\dots, m$ and compute the ASE of $\bbA[j]$ and of
\begin{gather*}
	\displaystyle \barbA_{(-j)}= \frac{1}{m-1}\sum_{\substack{t=1\\t\neq j}}^{m} \bbA[t],
\end{gather*}
the mean adjacency matrix over the left-out samples. We denote these ASEs as $\hbX_j$ and $\barbX_{(-j)}$, respectively.
Because $\var{\barbX_{(-j)} \barbX_{(-j)}^\top - \bbP} = \var{\hbX_j \hbX_j^\top - \bbP}/(m-1)$ as discussed in Remark \ref{rem:var_red} and~\cite{tang2018connectome}, we compute 
\begin{gather*}
	\bbE_j = \frac{\hbX_j\hbX_j^\top - \barbX_{(-j)} \barbX_{(-j)}^\top}{\sqrt{m-1}},
\end{gather*}
a fixed number of times, obtain a set of values $\bbE_j$, and estimate a ``worst-case'' $\hbE$ via the 0.99-quantile of this set. 

Note that the change detectability of the algorithm depends on the value of $\hbe$ and how close it is to $\bbe$. In particular, the relevant condition \eqref{eq:power} in practice becomes 
$\|\hbe\|^2<\|\bbe+\bbdelta\|^2_2.$\vspace{2pt}

\noindent \textbf{Finite memory statistics.} The CUSUM statistic $\Gamma[m,k]=\|\bbs[m,k]\|_2^2$ we have dealt with so far is based on the partial sum $\bbs[m,k] = \sum_{t=m+1}^{m+k}\bbh(\bbA[t],\hbX)$. As discussed in Remark \ref{rem:complexity}, it can be computed in a recursive and memory-efficient fashion that is ideal for online operation. Moreover, such an infinite-memory statistic accrues information from the entire data stream  $\{\bbA[m+k]\}_{k\geq 1}$, which is beneficial when it comes to invoking asymptotic approximations to the null distribution as in Section \ref{ssec:statistical_null}. However, if the change point $k_c$ occurs rather late during the monitoring horizon, then the inertia effect induced by a lengthy history of nominal graph observations will translate to longer detection delays. 

To attain faster reaction times one can resort to alternative finite memory statistics, which tend to rely on a judicious subset of the most recent observations. One natural variant is to adopt a fixed-length sliding window statistic, where the partial sum is $\bbs[k-L,k]$ for given window length $L$. At time $k$, this moving sum (MOSUM) statistic discards past data in the interval $(m,k-L)$, and its computation requires storing the last $L$ graphs in the sequence; see also \eqref{eq:mMOSUM} and~\cite{kirch2018sequentialCPD} for a modified version where the window length grows proportionally with $k$. Another useful procedure stems from the exponentially-weighted sum (EWSUM) statistic, namely
\begin{equation}\label{eq:ewsum}
\bbs_{\beta}[m,k] = \sum_{t=m+1}^{m+k}\beta^{m+k-t}\bbh\left(\bbA[t],\hbX\right),
\end{equation}
where $\beta\in(0,1]$ is a so-termed forgetting factor. EWSUM coincides with CUSUM for $\beta=1$, whereas for $\beta<1$ past samples are exponentially down-weighted and thus it offers a faster response to changes. Similar to CUSUM, \eqref{eq:ewsum} can be recursively updated as $\bbs_{\beta}[m,k]=\beta \bbs_{\beta}[m,k-1]+\bbh(\bbA[k],\hbX)$ and does not require storing any of the past measurements. Notice that as long as the window length is long enough  we may still use the results derived in Section \ref{ssec:statistical_null}, and the only algorithmic difference is that the weight $\omega[k]$ and the threshold $c_\alpha[k]$ should be changed accordingly (e.g., $\omega[k]=(r\min\{k,L\}^{3/2})^{-1}$ in the MOSUM case). The effect of choosing different windowed statistics is studied in the numerical tests of Section \ref{sec:simus}.

\section{Directed and Weighted Graphs}\label{sec:directed_weighted}

\subsection{Directed RDPG}\label{subsec:directed_rdpg}

As introduced in Section \ref{sec:rdpg}, the RDPG model is only suitable for undirected graphs. Indeed, $\bbX\bbX^\top=\bbP$ is always symmetric. For digraphs, edges (or arcs) are defined as \emph{ordered} pairs $(i,j)$, with $i,j\in \ccalV$. Since edges $(i,j)$ and $(j,i)$ are different objects, so could be the probabilities $P_{ij}$ and $P_{ji}$. By convention, we say $(i,j)$ starts from $i$ and points to $j$.\vspace{2pt} 

\noindent \textbf{Model specification.} Digraphs require an adaptation to the RDPG model, where each node $i\in \ccalV$ has an associated column vector $\bbx_i$ -- now in $\reals^{2d}$~\cite{priebe2017semiparametric}. Let us denote by $\bbx^l_i$ and $\bbx^r_i$ the first and last $d$ entries of $\bbx_i$, respectively. Likewise, let $\bbX^l,\bbX^r\in\reals^{N\times d}$ be the matrices stacking the transposed nodal vectors as their rows. In direct analogy to the undirected case, we define the directed RDPG (D-RDPG) model as
\begin{equation}\label{eq:drdpg_distribution}
\Pc{\bbA\given\bbX} = \prod_{i\neq j} [(\bbx_i^l)^\top\bbx_j^r]^{A_{ij}}[1-(\bbx_i^l)^\top\bbx_j^r]^{1-A_{ij}}
\end{equation}
[cf. the product over all $i\neq j$ here versus $i<j$ in \eqref{eq:rdpg_dist}], and the \emph{asymmetric} matrix of connection probabilities now becomes
\begin{gather}\label{eq:drdpg}
	\bbP = \bbX^l(\bbX^r)^\top.
\end{gather}
Intuitively, we say $\bbx_i^l$ models node $i$'s outgoing connectivity and $\bbx_i^r$ its incoming one. The probability of existence of the arc $(i,j)$ is given by $(\bbx^l_i)^\top\bbx^r_j$. 

Note that the rotational ambiguity is still present.
Actually, the ambiguity is exacerbated in this case because \emph{any invertible} matrix $\bbW\in\reals^{d\times d}$ will result in the same $\bbP$. Indeed, consider $\bbX^l\bbW$ and $\bbX^r\bbW^{-\top}$ and note that $\bbX^l\bbW(\bbX^r\bbW^{-\top})^\top=\bbX^l\bbW\bbW^{-1}(\bbX^r)^\top=\bbX^l(\bbX^r)^\top=\bbP$. Thus, as introduced the D-RDPG model \eqref{eq:drdpg} will be challenging to interpret, particularly when it comes to comparing two digraphs via their corresponding embeddings. This last task is critical when it comes to CPD. In order to have roughly the same level of ambiguity as in the undirected RDPG case, we will henceforth require that the $d$ columns of both $\bbX^l$ and $\bbX^r$ are orthogonal vectors (i.e., $(\bbX^l)^\top\bbX^l$ and $(\bbX^r)^\top\bbX^r$ are $d\times d$ diagonal matrices). This extra constraint does not fundamentally limit the expressiveness of the model because $\bbP$ is still of rank $d$.

All in all, we are left with the same rotational ambiguity as in the vanilla RDPG model, in addition to a scaling one. Indeed, consider a diagonal matrix $\textrm{diag}(\bbalpha)$ with non-zero entries and let $\bbW$ be an orthogonal  matrix. Then it follows that $\bbX^l\bbW\textrm{diag}(\bbalpha)$ and $\bbX^r\bbW\textrm{diag}(\bbalpha)^{-1}$ (which still have orthogonal columns) will produce the same $\bbP$ as \eqref{eq:drdpg}. Consequently, comparing the magnitude of $\bbx_i^l$ with that of $\bbx_i^r$ is meaningless. This scaling ambiguity, which to the best of our knowledge was overlooked before, will challenge CPD if one is interested in the behavior in a single direction (either incoming or outgoing). This is an interesting extension we will leave for future work. \vspace{2pt} 

\noindent \textbf{D-RDPG inference.} Let us now discuss how to estimate the matrices $\bbX^l$ and $\bbX^r$ from a graph observation. Since $\bbP=\E{\bbA}$ still holds, we seek a pair $\{\hbX^l,\hbX^r\}$ with orthogonal columns such that $\hbX^l(\hbX^r)^\top$ is the best rank-$d$ approximant of $\bbA$. Letting $\bbA=\bbU\bbD\bbV^\top$ be the singular-value decomposition (SVD) of $\bbA$, we set
\begin{gather}\label{eq:dase}
	\hbX^l = \hbU\hbD^{1/2}\text{ and }\hbX^r = \hbV\hbD^{1/2}. 
\end{gather}
Note that \eqref{eq:dase} satisfies the required orthogonality constraint. The choice in terms of scaling and counterscaling of columns is arbitrary. Choosing $\hbD^{1/2}$ assumes an even contribution of the incoming and outgoing connectivity of incident nodes to edge generation, which seems reasonable in lieu of any additional prior information. The scaling ambiguity is inconsequential to digraph CPD if we adopt the monitoring function $\bbH(\bbA[t],\{\hbX^l,\hbX^r\})=\hbX^l(\hbX^r)^\top-\bbA[t]$.

\subsection{Weighted RDPG}\label{subsec:weighted_rdpg}

We now shift our focus to weighted, undirected graphs. Let us then define a positive weight for each edge through a map $w:\ccalE\mapsto \reals_+$ such that $A_{ij}=A_{ji}=w_{ij}$ for $(i,j)\in\ccalE$. The absence of an edge is encoded as $A_{ij}=A_{ji}=0$. Naturally, an unweighted graph is a particular case of a weighted graph where the edge weights are $0$ or $1$ (i.e., $w\equiv 1$). 

A couple of works have proposed similar adaptations of the vanilla RDPG model to the weighted case; see~\cite{deford2016random,tang2017robust}. The basic ideas therein are outlined next. Suppose that the (possibly weighted) adjacency entries are generated from a given parametric distribution $F_{\bbtheta}(A_{ij})$ with $\bbtheta\in\reals^L$, for instance $\theta=\lambda$ for a Poisson$(\lambda)$ distribution. Each node $i\in \ccalV$ now has $L$ latent vectors $\bbx_i[l]\in \reals^{d_l}$ ($l=1,\ldots,L$), such that the weight $A_{ij}$ between nodes $i$ and $j$ is random with parametric distribution $F_{(\bbx_i^\top[1]\bbx_j[1],\ldots,\bbx_i^\top[L]\bbx_j[L])}(A_{ij})$, independently of all other edges. The distribution may have mass at $A_{ij}=0$ to capture sparse graphs where some pairs of nodes will be not be joined by edges. One recovers the vanilla RDPG by letting $F_{\bbtheta}(A_{ij})$ be a $\textrm{Bernoulli}(\theta)$ distribution. 

This approach has several drawbacks. For starters, all edges are required to have the same weight distribution, albeit with different parameters. This limitation may be partially overcome by considering a mixture distribution. However, and limiting even more its applicability, $F_{\bbtheta}(A_{ij})$ has to be chosen \emph{a priori}. So if edges have different weight distributions, we would have to know which of them adhere to each distribution (and what these distributions are) prior to inference. \vspace{2pt} 

\noindent \textbf{Model specification.} We propose instead that the sequence of vectors associated with each node is related to the moment generating function (MGF) of the weight distribution. 
In particular, each node has a sequence of latent positions $\bbx_i[l]\in\reals^{d_l}$ that determine the $l$-th moments of the weighted adjacency matrix as $\E{A_{ij}^l}=\bbx_i^\top[l]\bbx_j[l]$, for $l\in \naturals_+$.Given the sequence $\bbX:=\{\bbX[l]\}_l$, with $\bbX[l]=[\bbx_1[l],\ldots, \bbx_N[l]]^\top\in \reals^{N\times d_l}$, our weighted RDPG (W-RDPG) model specifies the MGF of the adjacency matrix as
\begin{equation}\label{eq:mgf}
	\E{e^{tA_{ij}}|\bbX} = \sum_{l=0}^\infty \frac{t^l\E{A_{ij}^l}}{l!} = 1+\sum_{l=1}^\infty \frac{t^l\bbx_i^\top[l]\bbx_j[l]}{l!}
\end{equation}
and the entries $A_{ij}$ are independent, i.e., edge independent. One can recover the vanilla RDPG by setting $\bbx_i[l]=\bbx_i\,\forall\: l$, where $\bbx_i$ is the vector associated to node $i$ in \eqref{eq:rdpg_dist}. \vspace{2pt}

\noindent \textbf{W-RDPG inference.}
Vectors $\bbx_i[l]$ are estimable via an ASE of matrix $\bbA^{(l)}$, where $\bbA^{(l)}$ denotes the entry-wise $l$-th power of adjacency matrix $\bbA$. The following theorem establishes the consistency (up to an unknown rotation) of this estimator, under some minor eigengap assumptions for the inner product matrices $\bbX[l]\bbX^\top[l]$.

\begin{mytheorem}\label{theorem:wrdpg}
Let $\bbB \in \reals^{N\times N}$ be a random, symmetric, and hollow matrix. Suppose that $0 \leq B_{ij} < M$ for some $M>0$ and that $\{B_{ij}\}_{i<j}$ are independent with $\E{B_{ij}} = P_{ij}$, where $\bbP=\bbX \bbX^\top$ for some fixed $\bbX \in \reals^{N\times d}$. Suppose that $\rank(\bbP) = d$ and that $\bbP$ has $d$ distinct positive eigenvalues $\lambda_1 > \lambda_2 > \ldots > \lambda_d>0$ that satisfy
\begin{gather*}
    \min_{i \neq j} |\lambda_i - \lambda_j| > \delta N  \text{ and } \lambda_d > \delta N
\end{gather*}
for some $\delta >0$.

Let $\hbX\in \reals^{N\times d}$ denote the ASE of $\bbB$, where it is assumed that the latent space dimension $d$ is known. Then almost surely there exists an orthogonal matrix $\bbW \in \reals^{d\times d}$ such that, for each $i \in \{1,\ldots,N\}$ and all $\gamma < 1$,
\begin{gather*}
    \mathbb{P}\left[||(\hbX\bbW)_{i\cdot} - (\bbX)_{i\cdot}||_2^2 > N^{-\gamma}\right] = o\left(N^{\gamma-1}\log N\right)
\end{gather*}
where $(\bbC)_{i\cdot}$ denotes the $i$-th row of matrix $\bbC$.
\end{mytheorem}

The relevance to W-DRPG inference is that Theorem~\ref{theorem:wrdpg} can be applied, for each fixed $l$, to $\bbB = \bbA^{(l)}$ to ensure that the latent position matrix $\bbX[l]$ can be consistently recovered (\textit{modulo} an orthogonal transformation) via the ASE of $\bbA^{(l)}$. 

Theorem \ref{theorem:wrdpg} is an extension of~\cite[Theorem 4.1]{sussman2014consistent}, so in Appendix \ref{ProofTheoremWRDPG} we sketch how the proof therein can be adapted to our setting. The main differences with the result in ~\cite{sussman2014consistent} are that in our setting, \textit{a}) latent positions are not random, and \textit{b}) entries $B_{ij}$ of matrix $\bbB$ are not necessarily Bernoulli random variables; we only assume that they are bounded and their expectation is given by the dot product of the corresponding latent positions. We remark that \textit{b}) is a more general setup than that of \cite{sussman2014consistent}. 
Extending the W-RDPG model to accommodate random latent positions remains an open direction we will pursue as part of our future work.



\begin{figure}
	\centering
	\includegraphics[width=0.15\textwidth]{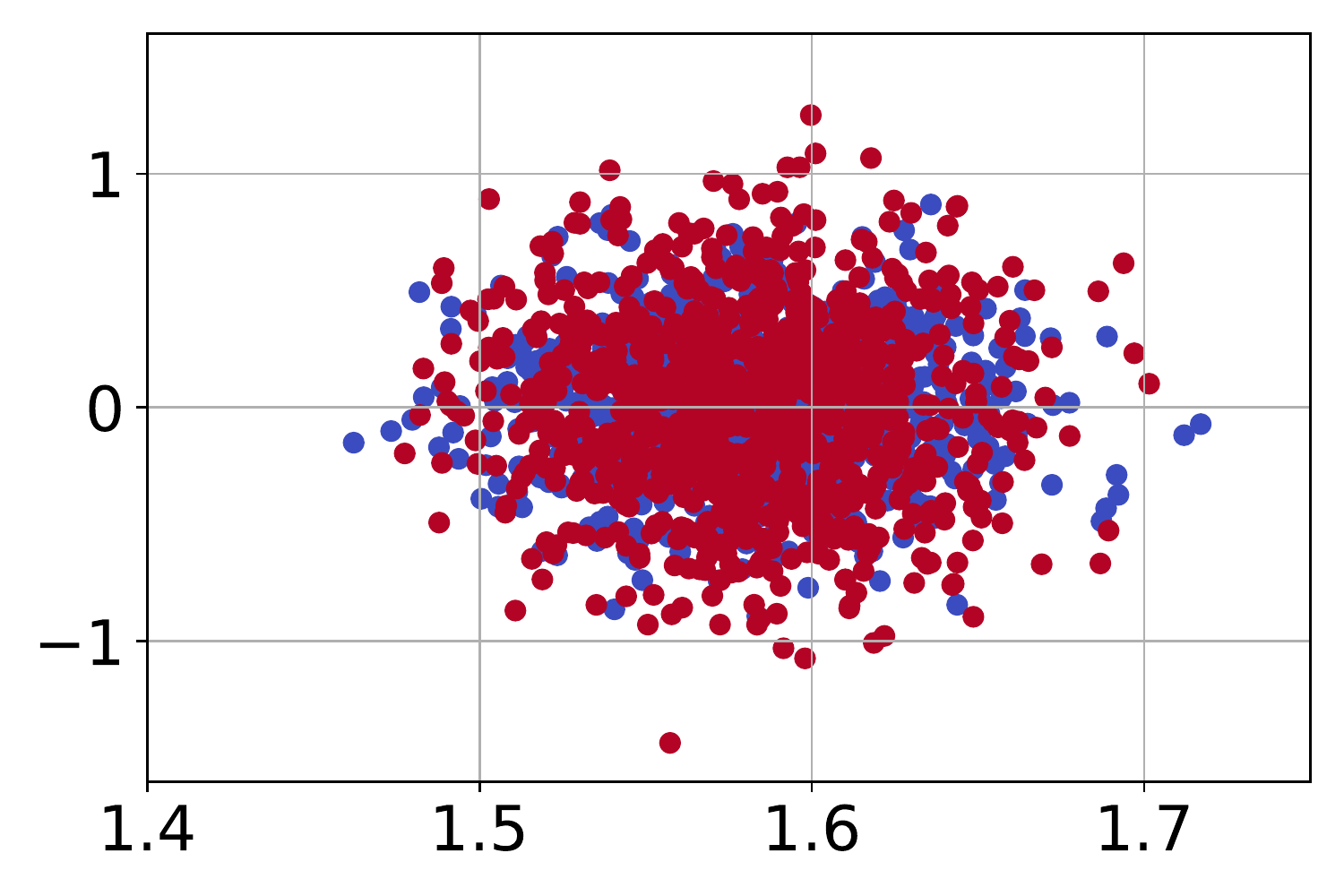}
	\includegraphics[width=0.15\textwidth]{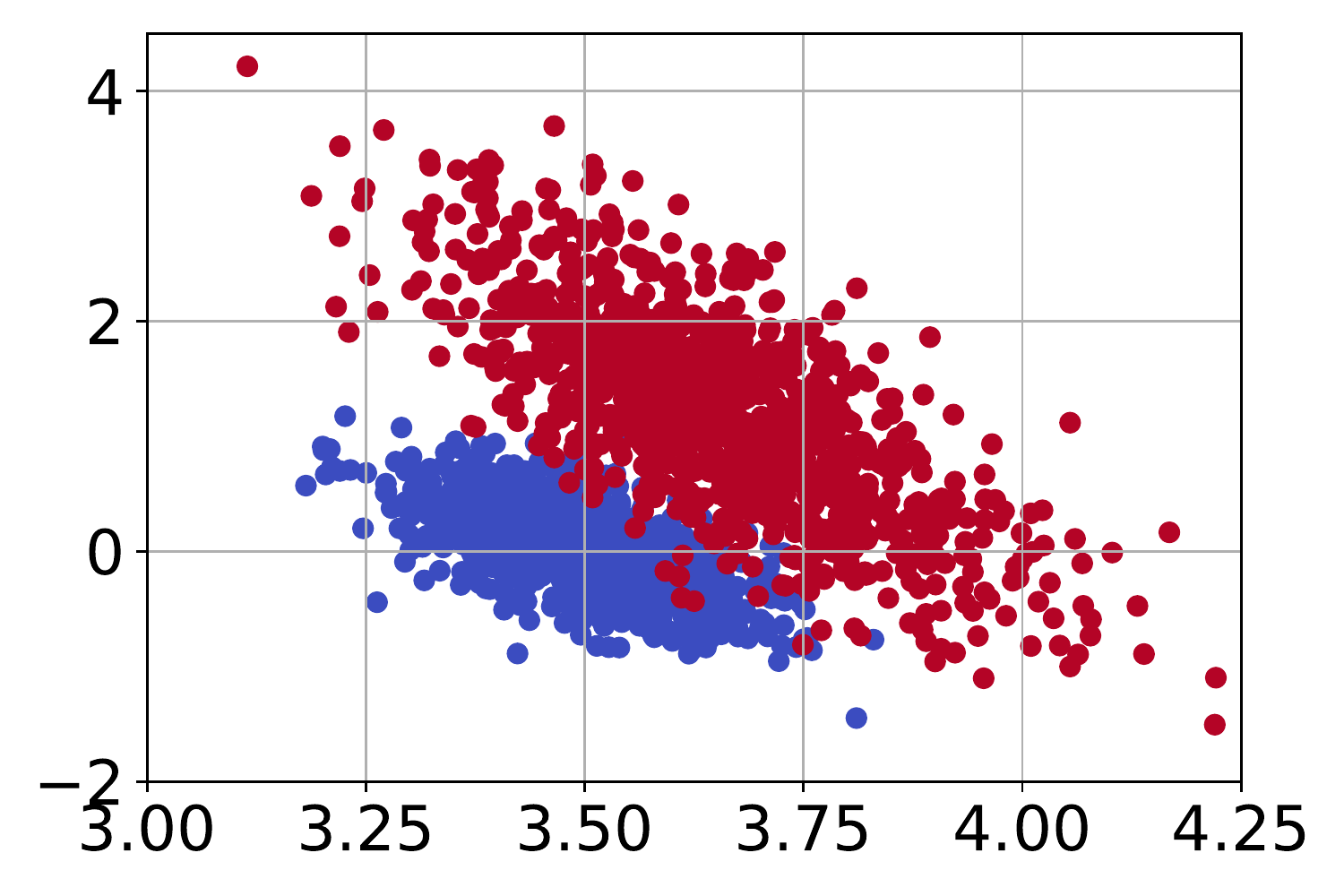}
	\includegraphics[width=0.15\textwidth]{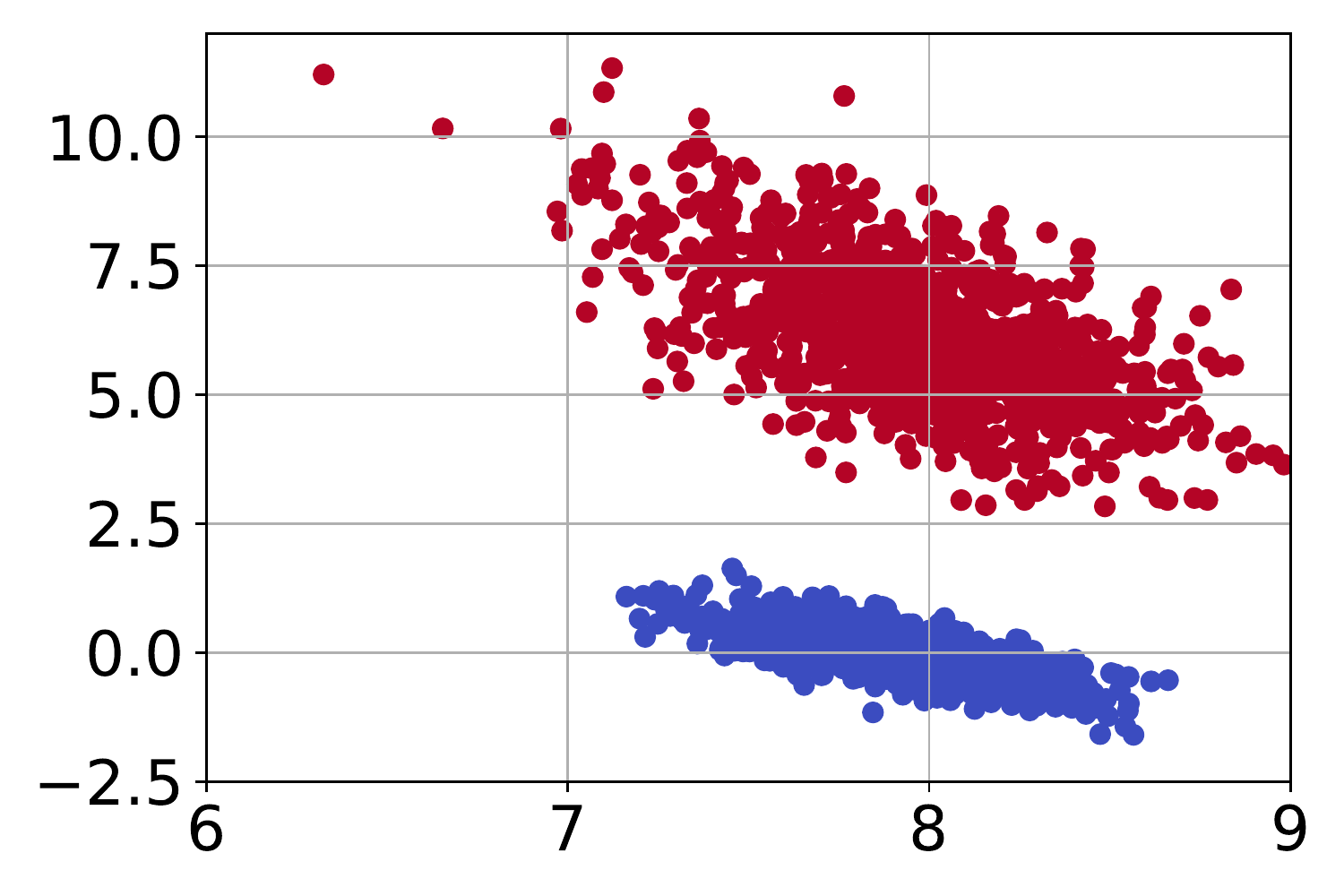}
	\caption{ASE embedding of $\bbA^{(l)}$ for Gaussian ($\mu=5$ and $\sigma=0.1$; in blue) and Poisson ($\lambda=5$; in red) distributed weights for $d_l=2$ and $l=1$ (left), $l=2$ (center), and $l=3$ (right). Nodes with different weight distributions are clearly revealed for $l=3$, but they overlap for $l=1$.}
	\vspace{-0.3cm}
	\label{fig:embedding_norm_poisson}
\end{figure}
\begin{myexample}\label{Ex:SBM}\normalfont
To illustrate the discriminative power of this novel embedding, we consider a two-block weighted SBM graph with $N=2000$ nodes. Edges are formed with fixed probability $p=0.5$, but weights are Gaussian with mean $\mu=5$ and standard deviation $\sigma=0.1$ for all edges except between a group of $1000$ nodes indexed as $i=1001,\ldots,2000$, where the weights' distribution is Poisson with parameter $\lambda=5$. 
As discussed in Example \ref{Ex:ER_SBM_Lantent}, in this case matrix $\bbX[l]$ will have at most $2$ different columns for all $l$. The vectors $\hbx_i[l]$ corresponding to the ASE for $l=1,2,3$ and $d_l=2$ are shown in Figure~\ref{fig:embedding_norm_poisson}, where each community is colored differently. 

Note how the nodes are indistinguishable for $l=1$. Indeed, the $\hbx_i[1]$ vectors are, as expected, centered around $(\sqrt{\mu p},0)=(\sqrt{\lambda p},0)\approx (1.58,0)$ corresponding to the mean weight. For $l=2$, Figure~\ref{fig:embedding_norm_poisson} (center) shows the vectors start to separate into the corresponding communities. Expressions for higher-order embeddings are easily obtained for this toy example. For instance, arbitrarily assuming that the $\bbx_i[l]$ lie on the abscissa for $i=1,\ldots,1000$ (recall that any rotation of $\bbX[l]$ will result in the same expected weights), it thus follows
%
%
%
\begin{equation*}
{\small
\bbx_i[2]=\left\{\begin{array}{ll}(\sqrt{p(\mu^2+\sigma^2)},0)& i\leq 1000,\\ (\sqrt{p(\mu^2+\sigma^2)},\sqrt{p(\lambda^2+\lambda-(\mu^2+\sigma^2)})& i> 1000.\end{array}\right.
}
\end{equation*}
Indeed, the estimates are around $(3.55,0)$ and $(3.55,1.58)$ respectively, although the noise corrupting the estimates hinders the ability to distinguish both distributions. For $l=3$, where the skewness of the distribution comes into play, vectors are clearly separated into the two groups; see Figure~\ref{fig:embedding_norm_poisson} (right).
\end{myexample}

\subsection{Online change-point detection}\label{ssec:ocpd_d/wrpg}

Let us briefly discuss how to perform online CPD for the general weighted and/or directed case. Extending the results presented in Section \ref{sec:online_cpd} to digraphs is straightforward. The only noteworthy difference is that, since the adjacency matrices are no longer symmetric, we need to consider entries from the entire residual matrix $\bbH$ (except the diagonal) during online monitoring, instead of the upper triangular half in \eqref{eq:vec_triu_h}. 

The path forward in the weighted case is also clear. The important difference is that the variance of each $A_{ij}$ is no longer of the form $p_{ij}(1-p_{ij})$, because we are naturally allowing for non-Bernoulli edge weight distributions. Following the W-RDPG model we introduced in the previous section, we have $\var{A_{ij}}=\bbx_i^\top[2]\bbx_j[2]-(\bbx_i^\top[1]\bbx_j[1])^2$. We rely on plugin variance estimates using the corresponding ASEs to compute the thresholds for the numerical test cases that follow [cf.\ vector $\bm\sigma$ in \eqref{eq:mean_Gamma_error} and \eqref{eq:var_Gamma_error}]. One can seamlessly blend the ideas in Sections \ref{subsec:directed_rdpg} and \ref{subsec:weighted_rdpg} to perform online CPD for weighted digraphs. The provided code offers this functionality.

In closing, note that the aforementioned discussion is pertinent only when the goal is to detect changes in the mean adjacency matrix (i.e., $l=1$). This is the scope of the ensuing numerical experiments. Considering larger values of $l$ could be prundent when interested in more fine-grained changes on the weights' distribution, as illustrated in Example \ref{Ex:SBM}. 

\section{Numerical Experiments}\label{sec:simus}

Here we carry out numerical experiments to evaluate the performance of the proposed online CPD algorithm for weighted and (un)directed graph sequences. We start with a controlled synthetic data setting, where the goal is to identify emergent network community structure (Section \ref{ssec:synthetics}). We carefully examine: (i) the choice of the detection threshold and monitoring function; (ii) the choice of the running statistic and its effect on the detection delay; (iii) robustness to the prescribed false alarm rate $\alpha$; and (iv) comparisons with relevant batch and online CPD methods. Test cases with real wireless and social network data are presented in Section \ref{ssec:real_data}. For the implementations we used the Python libraries \texttt{NumPy}~\cite{numpy}, \texttt{NetworkX}~\cite{networkx}, \texttt{pandas}~\cite{pandas}, \texttt{graspologic}~\cite{graspy}, as well as our own code which we share in \url{https://github.com/git-artes/cpd_rdpg}. For the comparison with the online CPD method in~\cite{chen2019sequentialCPD}, we used the official \texttt{R} implementation in the \texttt{gStream} package with the default parameters settings. 
Furthermore, as a baseline we have implemented the offline CPD algorithm described in~\cite{padilla2019change}. 
This implementation is also available in our GitHub repository.  

\subsection{Simulated data}\label{ssec:synthetics}

A timely problem is to detect when communities arise in networks. So, we first test the proposed online CPD method by generating a sequence of 150 ER graphs with $N=100$ nodes and connection probability $p=0.5$. After $t_c=150$, the model shifts to a two-block SBM with $N/2=50$ nodes in each community and connection probability $q_1=0.6$ for nodes in the same community and $q_2=0.4$ for nodes in different blocks. We use the first $m=50$ graphs as the training set, and the value of $d$ is automatically chosen (via scree plot) by the \texttt{graspologic} library used to obtain the ASE. Because the index $k$ in $\Gamma[m,k]$ measures how much time has elapsed since monitoring started, the change-point is at $k_c=100$.

\begin{figure}[t!]
	\centering
	\includegraphics[width=0.9\linewidth]{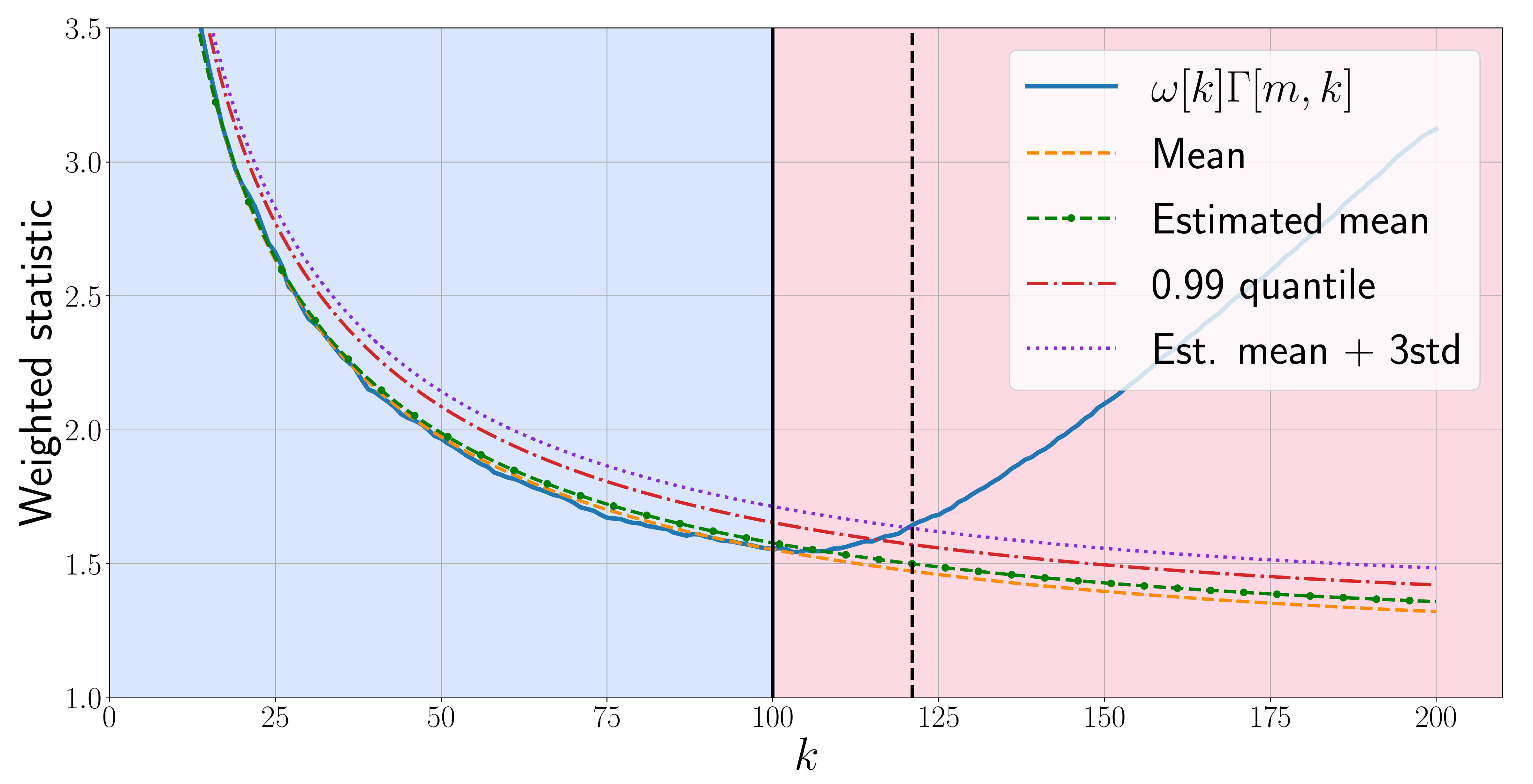}
	\caption{Evolution of $\omega[k]\Gamma[m,k]$, its mean and the estimated mean, for simulated data. Two thresholds are shown: the 0.99-quantile of the distribution in~\eqref{eq:ER_nonperfect} and three standard deviations away from the mean; those thresholds are very close and the latter is preferred due to its reduced complexity. The solid vertical line indicates the actual change-point, while the dashed one is the detection. A change in background color indicates a change-point detected by the offline algorithm~\cite{padilla2019change}. Our approach is able to detect the change with a relatively small delay, while operating in an online fashion.}
	\label{fig:ej_sintetico}
\end{figure}

Figure \ref{fig:ej_sintetico} depicts the results for this test case. We show two thresholds: the 0.99 quantile of the estimated distribution [i.e., the distribution given by~\eqref{eq:ER_nonperfect} but with $\hbe$ instead of $\bbe$] and $\text{th}[k]$, the estimated mean plus three standard deviations. Apparently, the difference between those two thresholds is small, so $\text{th}[k]$ is preferred due to its reduced complexity. Using that threshold a change-point is declared at $k^*=121$, so our algorithm is successfully identifying the change in the model. The detection delay can be explained if we look at the estimated mean of the weighted CUSUM statistic. Since we are estimating the error $\bbE$ as the 0.99-quantile over the training set, we always overestimate the true value. Also, since we are monitoring the cumulative sum \eqref{eq:cusum}, if a change occurs after a long period of time then the drift in  $\Gamma[m,k]$ will not be noticed immediately; see also the discussion in Section \ref{ssec:imp_details}. As a way to compare the performance of Algorithm \ref{A:online_CPD} with other approaches, Figure \ref{fig:ej_sintetico} also shows the detection result for the offline baseline proposed in~\cite{padilla2019change}. That algorithm detects the change with no delay, but it has a markedly greater computational complexity than ours and examines the entire data sequence as a batch.\vspace{2pt}

\begin{figure}[t!]
	\centering
	\includegraphics[width=0.9\linewidth]{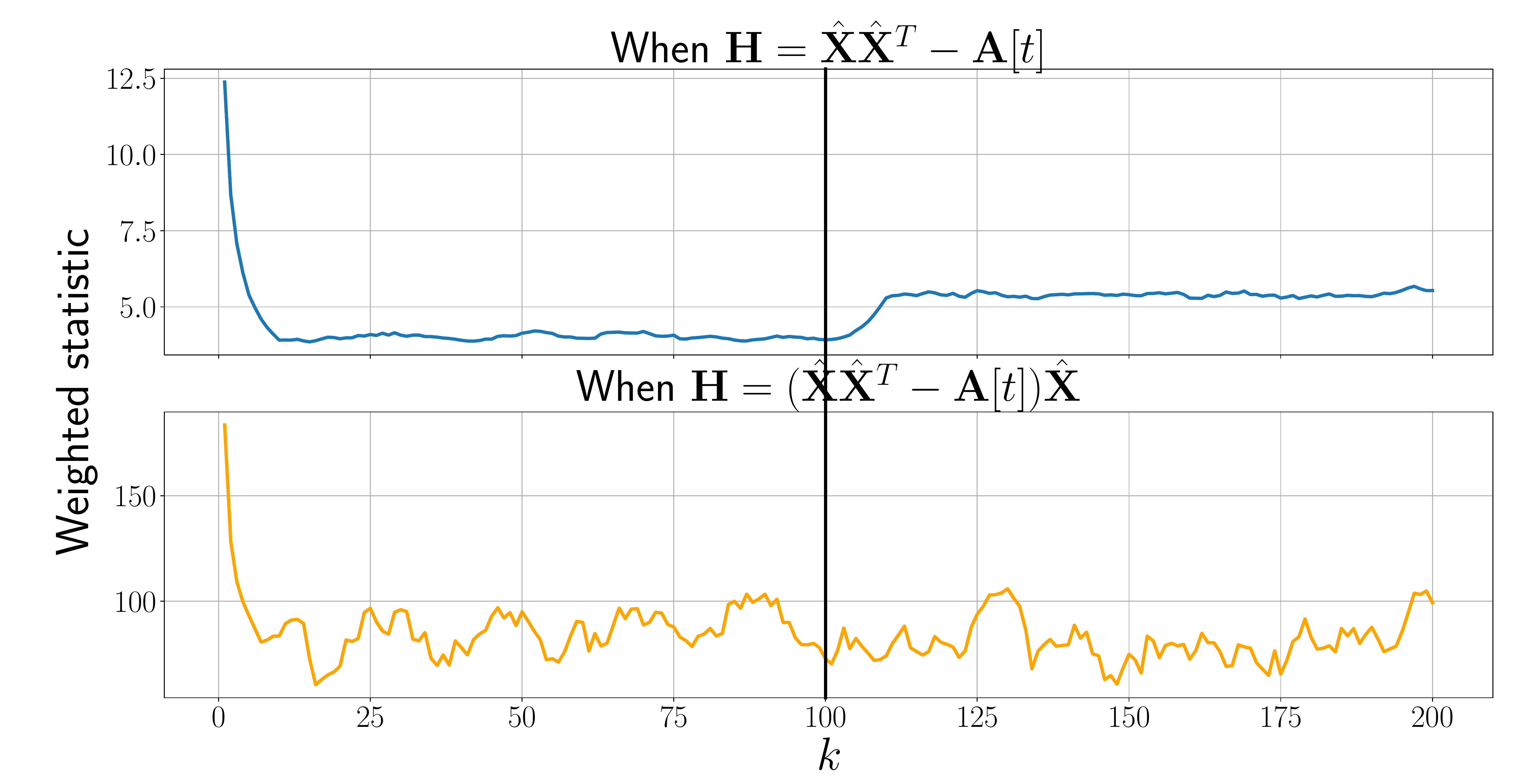}
	\caption{Evolution of $\omega[k]\Gamma[m,k]$ for residual (top) and projection (bottom) monitoring functions, using the MOSUM sliding window statistic. After the change-point there is a discernible change in trend for the residual; the projection does not exhibit such desirable behavior.}
	\label{fig:resid_proyect}
\end{figure}

\noindent \textbf{On the choice of the monitoring function.} In this running example, had we used the monitoring function $\bbH'=\left(\hbX\hbX^\top-\bbA[t]\right)\hbX$ (i.e., use the projection instead of the residual) we would have missed the change altogether. Indeed, for perfect ASE estimation, if our training data adheres to an ER model with parameter $p$ then $\hbX = \sqrt{p}\bbJ_{N\times d}$ and $\hbX\hbX^\top = p\bbJ_N$, with $\bbJ_{N\times d}$ denoting the ${N \times d}$ all-ones matrix. Now suppose there is a change in the nominal model and we shift to a two-block SBM, where each community has $N/2$ nodes and the connection probabilities are $q_1$ for nodes in the same block and $q_2$ for nodes in different communities. 
The connection probability matrix for said SBM is
\begin{gather}
	\bbP_{\text{SBM}} = \left(
	\begin{array}{c|c}
		\bbQ_1 & \bbQ_2\\
		\hline
		\bbQ_2 & \bbQ_1
	\end{array}
	\right),
	\label{eq:P_sbm}
\end{gather}
where $\bbQ_1=q_1(\bbJ_{N/2}-\bbI_{N/2})$ and $\bbQ_2=q_2\bbJ_{N/2}$, with $\bbI_{m}$ denoting the identity matrix of size $m$. After the change we thus have $\E{\bbH'} = (p\bbJ_N - \bbP_{\text{SBM}})\sqrt{p}\bbJ_{N\times d}$. Since each row of $\bbP_{\text{SBM}}$ has $N/2-1$ entries with value $q_1$ and $N/2$ entries with value $q_2$, each entry of $\E{\bbH'}$ is given by
\begin{align*}
	\left(\E{\bbH'}\right)_{ij} &= \left(\frac{N}{2}-1\right)\sqrt{p}(p-q_1) + \frac{N}{2}\sqrt{p}(p-q_2) \\
	&\approx N\sqrt{p}\left(p- \frac{q_1+q_2}{2}\right),
\end{align*}
for large $N$. Accordingly, choosing $p$, $q_1$ and $q_2$ such that $q_1+q_2=2p$ (as was the case for our simulation), we find that $\E{\bbH'} = \bb0$, i.e. we do not expect to see a drift in the monitoring function after the change. 

Figure \ref{fig:resid_proyect} shows the evolution of the weighted statistic $\omega[k]\Gamma[m,k]$ for both choices of the monitoring function. The setup is the same as in the previous test case, with a change-point located at $k_c=100$. The MOSUM statistic is adopted here, using a sliding window of length $L=10$. When the residual $\bbH$ is chosen as the monitoring function, a sudden shift in trend is observed after the change-point. However, when the projection $\bbH'$ is used the statistic does not exhibit such desirable behaviour and misses the model change.\vspace{2pt}

\noindent \textbf{On the sensitivity to $\alpha$.} Here we examine the robustness of Algorithm \ref{A:online_CPD} to the choice of the false alarm rate $\alpha$. We simulate the same scenario as before, except that $N=20$ in order to increase the variance of $\omega[k]\Gamma[m,k]$ and the  error $\bbE$. As thresholds we test $c_\alpha[k]$ for $1-\alpha\in\{0.99, 0.95, 0.9\}$, along with two versions of $\textrm{th}[k]$: the mean plus two and three times the standard deviations as in \eqref{eq:threshold_simple}. 

The results are depicted in Figure \ref{fig:ej_sintetico_varios_percentiles}. The example illustrates how using $1-\alpha=0.95$ or $1-\alpha=0.9$ may prove too conservative. In this particular instance, $1-\alpha=0.9$ would result in a (false) change-point detected at $k\approx 10$. Furthermore, both versions of $\textrm{th}[k]$ provide reasonable results, although the one that uses three standard deviations is consistently above $c_{0.01}[k]$ and is thus preferred. We will re-examine this choice in Section \ref{ssec:real_data}, when we present real-world examples.\vspace{2pt}

\begin{figure}
    \centering
    \includegraphics[width=0.9\linewidth]{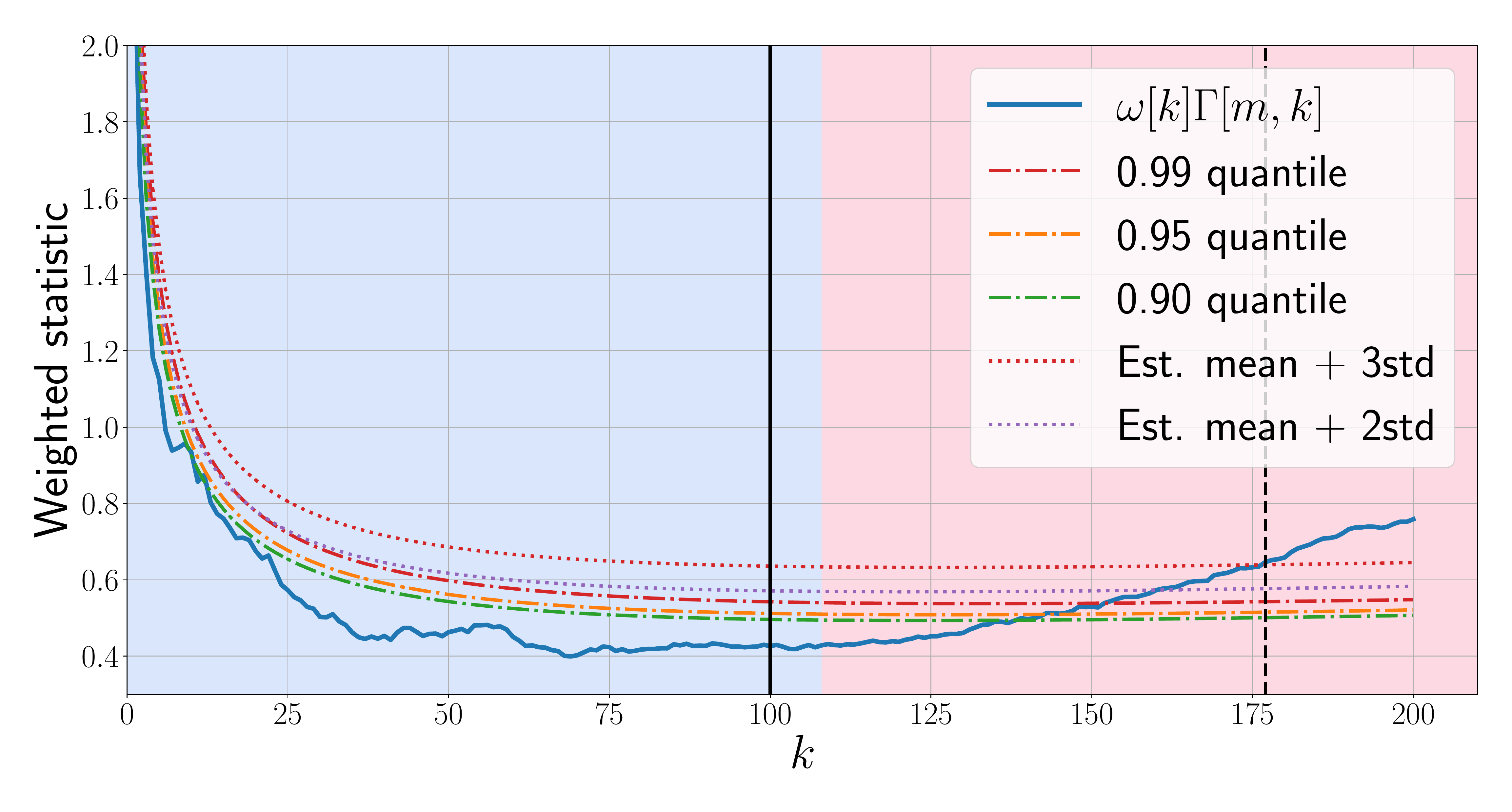}
    \caption{Evolution of $\omega[k]\Gamma[m,k]$ and five possible thresholds: $c_\alpha[k]$ (for $1-\alpha\in\{0.9, 0.95, 0.99\}$) and $\textrm{th}[k]$ equal to the mean plus two and three standard deviations. The setting is the same as in Fig.\ \ref{fig:ej_sintetico} except that $N=20$ to increase the variance of $\omega[k]\Gamma[m,k]$. Using $1-\alpha=0.99$ is preferred as it provides more robustness to false positives. Both choices of $\textrm{th}[k]$ are reasonable, although using three standard deviations is consistently above $c_{0.01}[k]$ (see the first time-steps).}
    \label{fig:ej_sintetico_varios_percentiles}
\end{figure}

\noindent \textbf{Comparison with~\cite{chen2019sequentialCPD}.} An online CPD algorithm based on a $k$-nearest neighbor approach was proposed in~\cite{chen2019sequentialCPD}. Observations are viewed as points in a normed space and the distance induced by such norm is used to define a neighborhood for each observation. Changes are detected by performing two-sample testing on the neighborhood graph. The proposed approach is computationally intensive because it requires that, if the current observation index is $n$, a two-sample hypothesis test is performed for each time $t \in \{1,\dots,n-1\}$ (or for a subset of these time instants). Also, it is memory-inefficient since one has to store the pairwise distances between all past observations. Even if these aspects are not a concern, this approach is ill-suited to detect changes in some sequences of networks, as we will argue shortly.

An example in~\cite{chen2019sequentialCPD} illustrates the performance of the CPD algorithm on network sequences. Observations are the adjacency matrices of the graphs and neighborhoods are defined using the distance induced by the Frobenius norm over such matrices. We will see that this distance does not allow for capturing some changes in the network connectivity, such as the formation of two communities discussed so far. Indeed, if $\bbA,\bbB \in \mathbb{R}^{N\times N}$ are adjacency matrices of two ER graphs with connection probability $p$, then
\begin{gather*}
	\E{\|\bbA - \bbB\|_F^2} 
	=N(N-1)2p(1-p),
\end{gather*}
since all entries $A_{ij},B_{ij}\sim\textrm{Bernoulli}(p)$. Thus $\|\bbA - \bbB\|_F^2 \approx 2p(1-p)N^2$ for sufficiently large $N$. Suppose now that $\bbC$ and $\bbD$ are two adjacency matrices from a two-block SBM, where each community has $N/2$ nodes and the connection probabilities are $q_1$ for nodes in the same cluster and $q_2$ for nodes in different communities. Then the connection probability matrix for $\bbC$ and $\bbD$ is given by~\eqref{eq:P_sbm}, so those matrices have $N^2/2$ entries whose expected value is $q_2$ and $(N/2-1)N\approx N^2$ entries whose expected value is $q_1$. All in all, similarly to the ER case we have
\begin{align*}
	\|\bbC - \bbD\|_F^2 &\approx N^2 \left( q_1 -q_1^2 +q_2-q_2^2\right),\\
	\|\bbA - \bbC\|_F^2 &\approx N^2 \left( p -p (q_1+q_2) + \frac{q_1+q_2}{2}\right).
\end{align*}
Again, if we choose $q_1$ and $q_2$ such that $q_1+q_2=2p$, then we obtain $\|\bbA - \bbB\|_F^2 \approx \|\bbA - \bbC\|_F^2$. In other words, the distance between an observation before the change ($\bbA$) and an observation after the change ($\bbC$) will be very similar to the distance between two observed matrices before the change ($\bbA$ and $\bbB$). For matrices after the change, we have that when $q_1+q_2=2p$ then
\begin{align*}
	\|\bbC - \bbD\|_F^2 &\approx 2N^2 \left( p-p^2 -(p-q_1)^2\right),
\end{align*}
so choosing $p$ and $q_1$ to be very similar (but not equal, so there is effectively a change), for large $N$ these two models will be indistinguishable under the Frobenius distance criterion.

\begin{figure}[t!]
	\centering
	\includegraphics[width=0.9\linewidth]{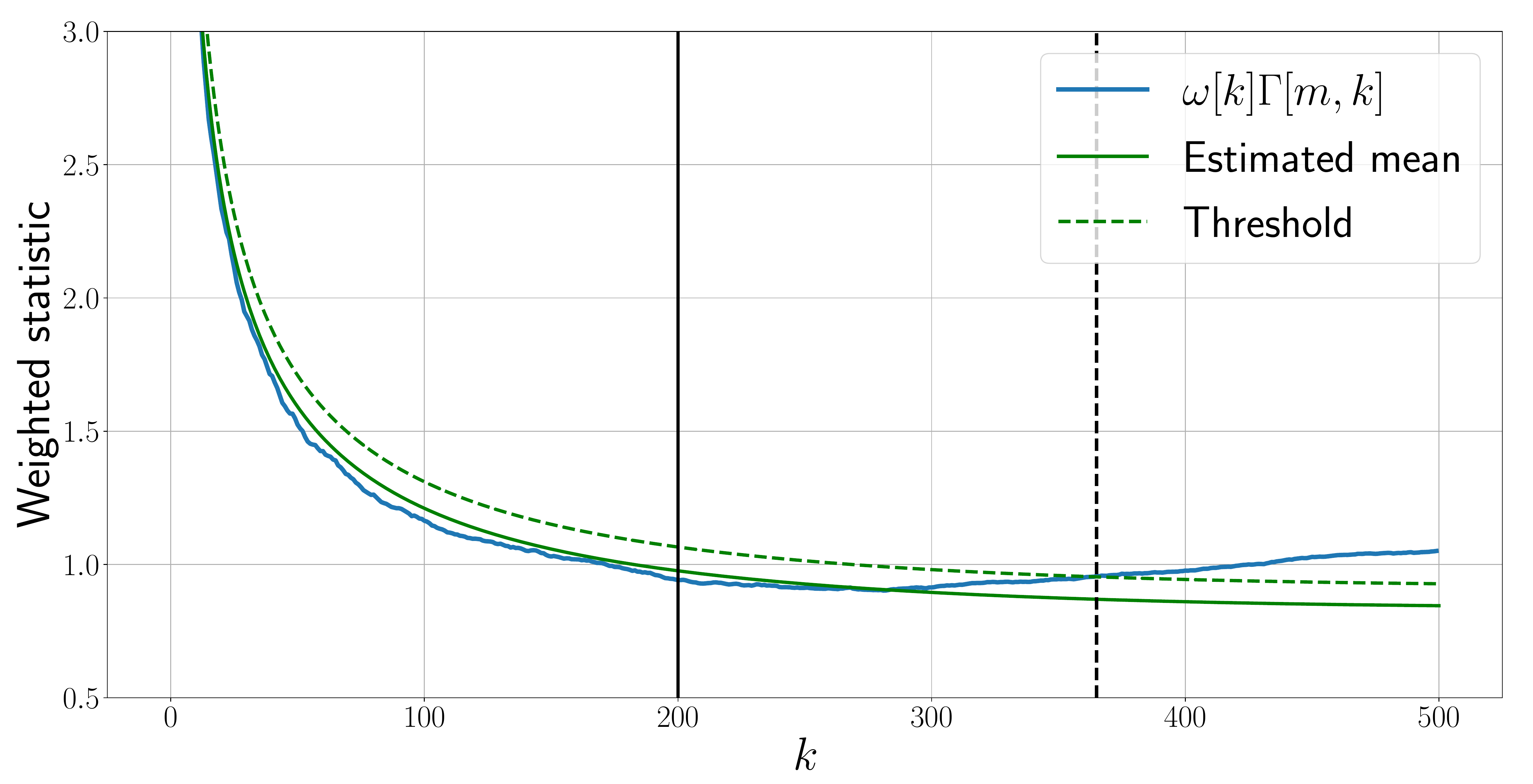}
	\caption{Detection result for a network transitioning from an ER model with $p=0.3$ to a two-block SBM with $q_1=0.275$ and $q_2=0.325$. Algorithm \ref{A:online_CPD} is able to detect the change in this setup, while the approach proposed in~\cite{chen2019sequentialCPD} fails to do so.}
	\label{fig:contraej_chen}
\end{figure}

We simulated such a setup, with a network of $N=100$ nodes switching from an ER model with $p=0.3$ to a two-block SBM with $q_1=0.275$ and $q_2=0.325$. The change-point was located at $k_c=200$. We ran the algorithm proposed in~\cite{chen2019sequentialCPD} using the implementation in the \texttt{R} package \texttt{gStream}. Selecting between 3 and 10 nearest neighbors and an average run length of 1000, it found no change-points in the data. Results for our CUSUM detector are depicted in Figure~\ref{fig:contraej_chen}. Apparently, there is a noticable change in trend in the weighted statistic after $k=300$, with a change-point being detected at $k^*=365$. This arguably large detection delay can be shortened using a finite memory statistic such as MOSUM.\vspace{2pt}
\begin{figure}[t!]
	\centering
	\includegraphics[width=0.9\linewidth]{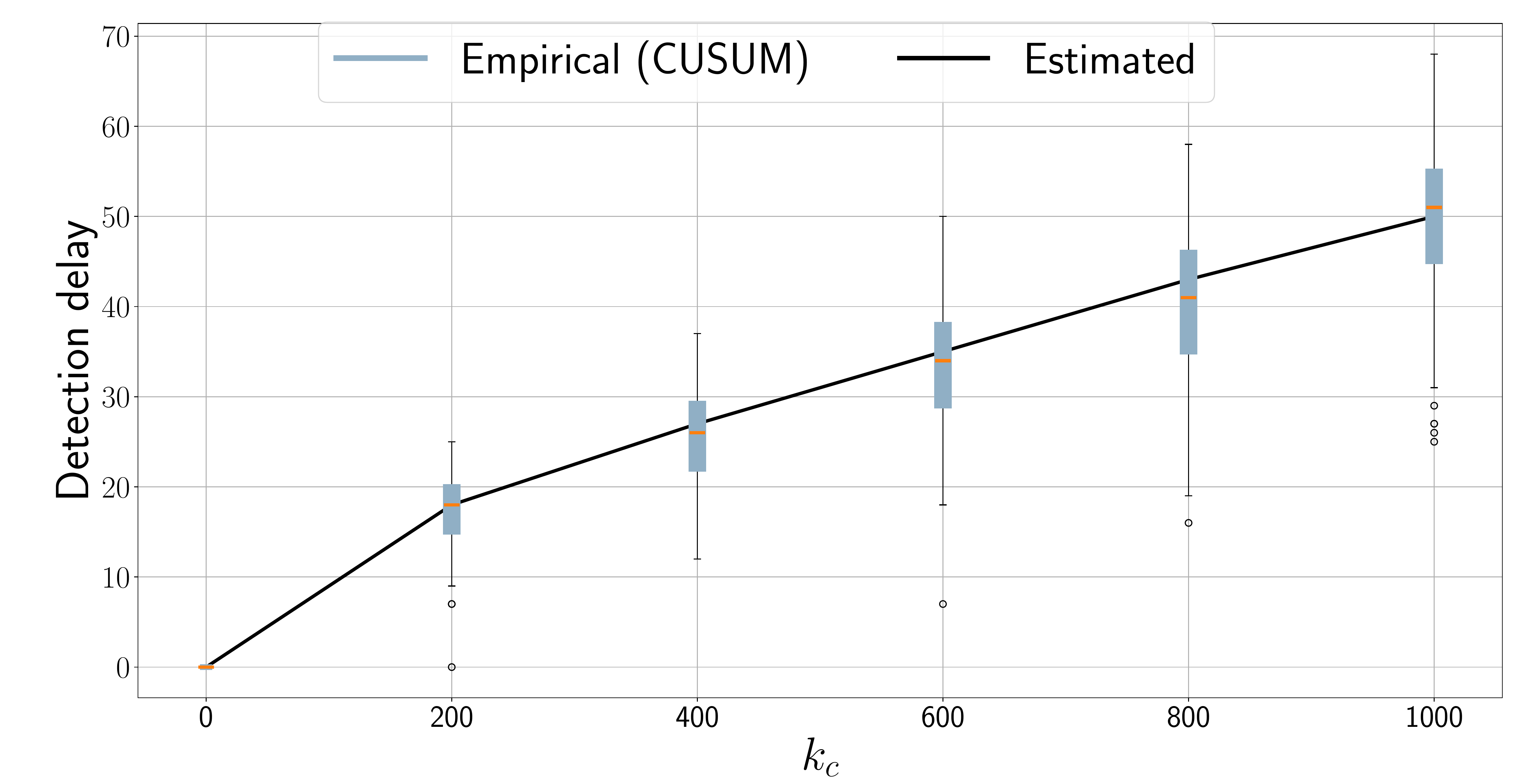}
	\caption{Estimated detection delay and empirical delays for different change-point locations $k_c$. Empirical delay is well predicted by the estimated curve. For the adopted CUSUM statistic, as expected the delay grows with $k_c$.}
	\label{fig:theoretical_delay}
\end{figure}

\noindent \textbf{Detection delay.} Characterizing the distribution of the detection delay $\tau$ (i.e., the time interval between the occurrence of a change and it actually being detected) is in general challenging. Instead, we will settle with a point estimate obtained via identification of the first instant the weighted statistic $\omega[k]\Gamma[m,k]$ crosses the threshold function $c_\alpha[k]$. Recall that this is the condition that defines the rejection region of our test. Since that statistic has finite variance, it is possible to predict at which time point $k^*$ the change will be detected by studying when the expectation of the weighted test statistic after the change [cf. \eqref{eq:mean_after_change}] first exceeds the threshold.
Once more, for simplicity and analytical tractability we will henceforth assume the threshold is set as $\textrm{th}[k]$ in \eqref{eq:threshold_simple}. This choice (approximately) corresponds to $\alpha=0.01$; see Figures \ref{fig:ej_sintetico} and \ref{fig:ej_sintetico_varios_percentiles} for further discussion on this point. 
To estimate the delay, we find the first instant $k^* \geq k_c$ for which $\omega [k^*]\mathbb{E}_{ac}[k^*] \geq \text{th}[k^*]$, where $\mathbb{E}_{ac}$ denotes the expectation of $\Gamma[m,k]$ after the change that is approximately given by~\eqref{eq:mean_after_change}. This amounts to solving the equation
\begin{gather}
    (k^*-k_c)^2 \left(  \|{\bm\sigma_Y}\|_1 - \|{\bm\sigma_X}\|_1+ 2k^*(\bbe^\top\bbdelta)+ (k^*-k_c) ||\bbdelta||_2^2 \right)^2 \nonumber\\
    = 9\left(2(k^*)^2\|{\bm\sigma_X}\|_2^2 + 4(k^*)^3(\bm\sigma_X^\top \bbe^2)\right),
    \label{eq:delay_inequality}
\end{gather}
which entails finding the roots of a fourth-order polynomial. The solution $k^*$ can be obtained numerically, and the estimated delay becomes $\tau = k^*-k_c$.

To test said method, we simulated a sequence of ER networks with $N=100$ nodes and connection probability $p=0.5$. We use the first $m=100$ graphs for training. The first $k_c$ graphs after training follow that same model, but then observations shift to an ER with $p=0.6$. The solution to~\eqref{eq:delay_inequality} allows us to estimate the detection delay for different values of $k_c$. This can be done after training, since once that phase ends the error $\bbe$ is fixed, and vectors $\bm\sigma_X$, $\bm\sigma_Y$ and $\bbdelta$ are defined by the change in the underlying model. Figure~\ref{fig:theoretical_delay} shows the estimated delay for $k_c \in \{0,200,400,600,800,1000\}$. For each $k_c$ a box plot of Algorithm \ref{A:online_CPD}'s empirical delays is also shown, computed for 100 simulated runs using the CUSUM statistic. Our estimation is consistent with the experimental delays in Algorithm \ref{A:online_CPD}, which tend to show a linear growth with $k_c$. 
\begin{figure}[t!]
	\centering
	\includegraphics[width=0.9\linewidth]{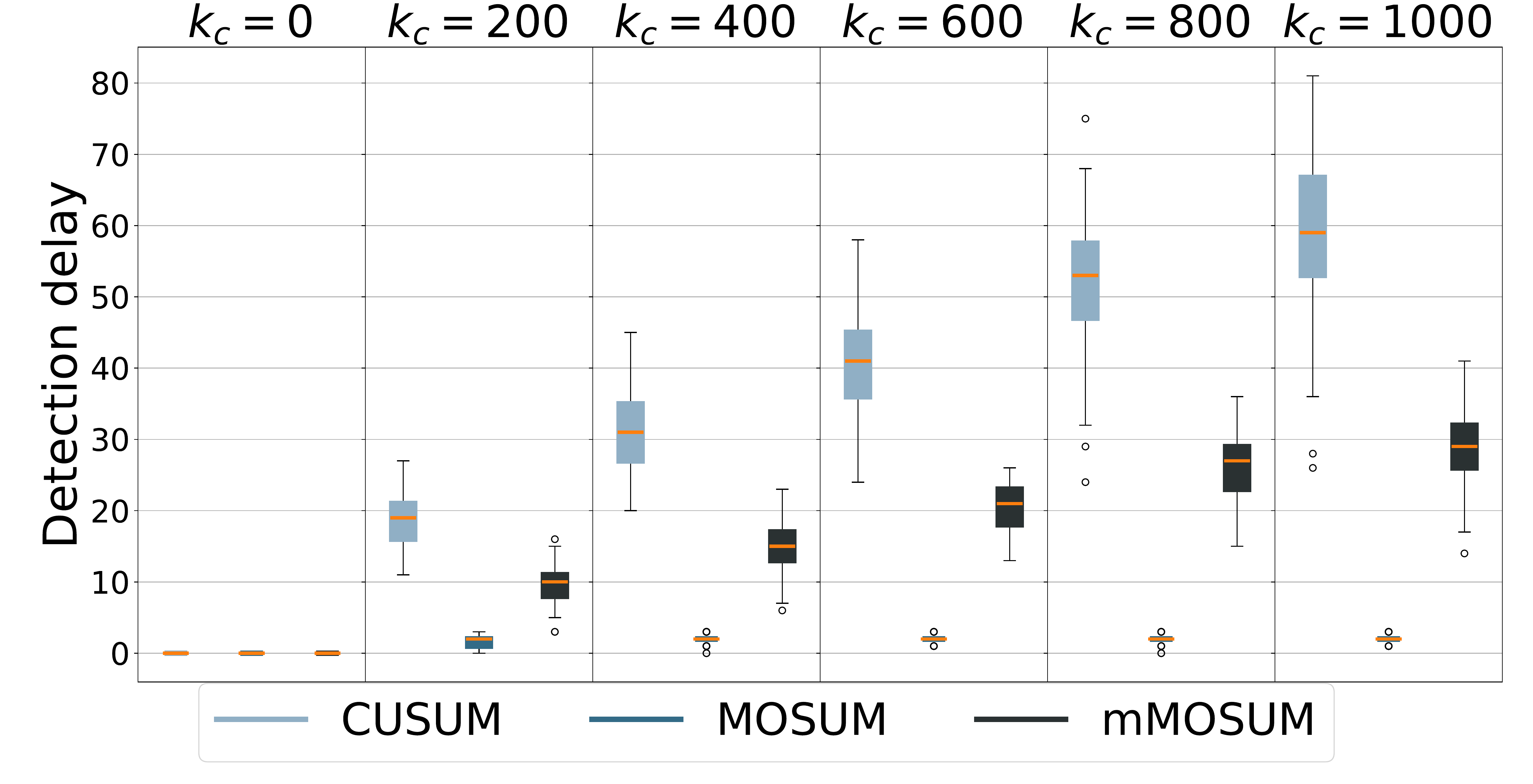}
	\caption{Empirical delays for different change-point locations $k_c$, using the CUSUM, MOSUM (with $L=10$), and mMOSUM (with $h=0.4$) statistics. Delays behave as expected given the different effective observation intervals: roughly constant delay for MOSUM, growing delays with $k_c$ for both CUSUM and mMOSUM, but at a slower rate for the latter.}
	\label{fig:delay_vs_k_c}
\end{figure}
\begin{figure}[t!]
	\centering
	\includegraphics[width=0.9\linewidth]{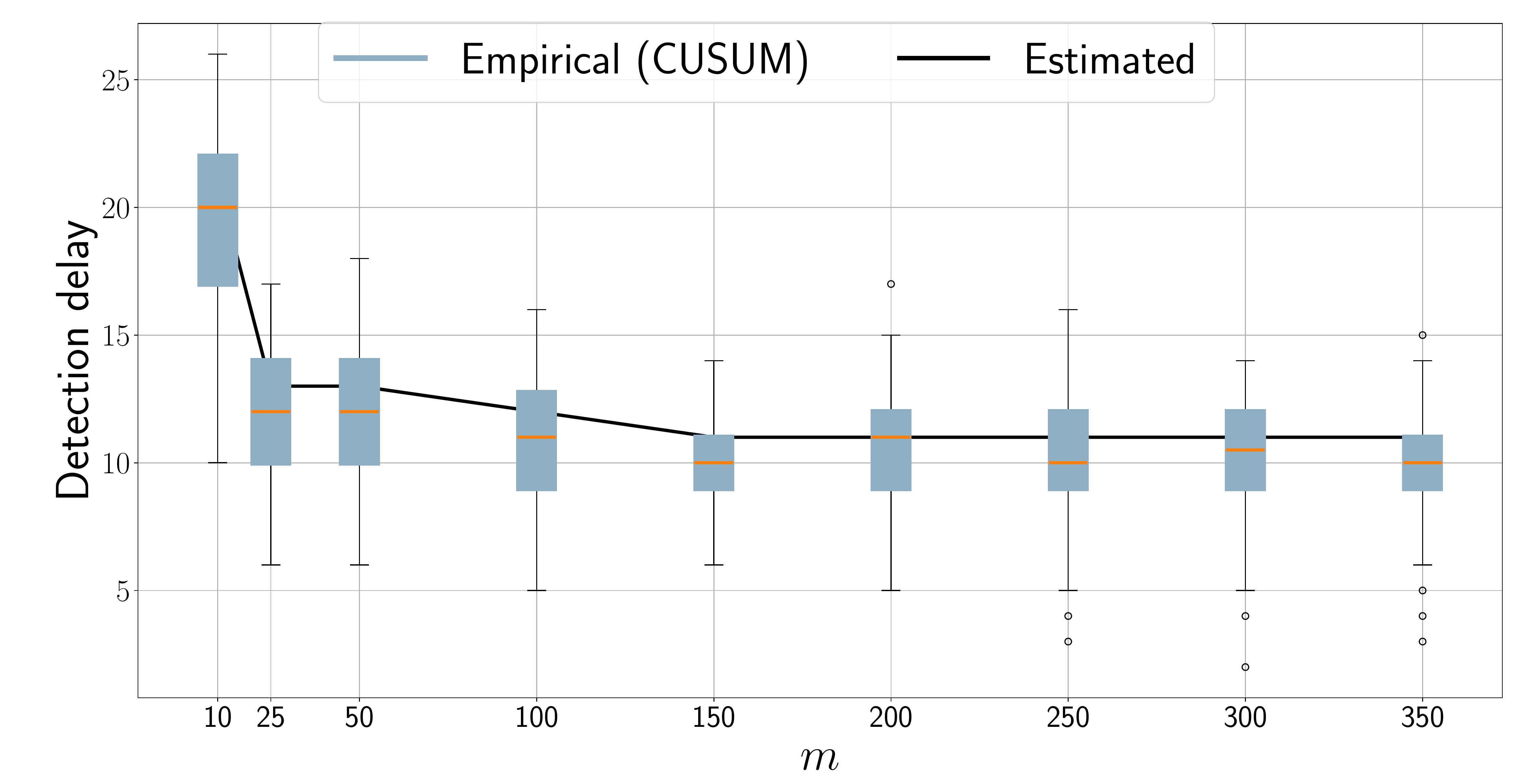}
	\caption{Estimated detection delay and empirical delays for different training set sizes $m$, for the CUSUM statistic. The delay is lower as $m$ increases, but there is no significant improvement after $m=25$.}
	\label{fig:delay_vs_m}
\end{figure}

Figure~\ref{fig:delay_vs_k_c} depicts the empirical delays in this setup for three different statistics: CUSUM, MOSUM and mMOSUM. This last running statistic is defined in~\cite{kirch2018sequentialCPD} as
\begin{equation}\label{eq:mMOSUM}
\bbs[m,k] = \sum_{t=m+\lfloor kh\rfloor+1}^{m+k}\bbh\left(\bbA[t],\hbX\right),
\end{equation}
where $h\in(0,1)$ and $\lfloor x\rfloor$ is the floor function, i.e., the largest integer that is smaller or equal to $x$. The mMOSUM is defined in a way such that early observations are discarded and the window length grows proportionally with $k$. Hence, the algorithm's response time should be faster than when using the CUSUM statistic. That is consistent with Figure~\ref{fig:delay_vs_k_c}, which shows that the detection delay for the mMOSUM statistic grows with $k_c$, but at a slower rate than that of CUSUM. For this simulation we set $h=0.4$. The MOSUM statistic, with a window length of $L=10$ observations, attains the shortest delay among the three and it is roughly constant with $k_c$. This is expected given that the window size remains constant for MOSUM, so there is no inertia associated with the change-point occurring long after monitoring started.

Finally, Figure~\ref{fig:delay_vs_m} shows the empirical and estimated delays for the CUSUM statistic for various training set sizes $m$. The setup is similar to that of the previous test case, with an ER model switching from $p=0.5$ to $p=0.6$ at $k_c=100$. As expected, the delay decreases with $m$, since more training samples lead to more accurate ASE estimates. Also, it is important to note that Algorithm \ref{A:online_CPD} performs well with a relatively small training set size. In this setting, we observe there is no significant improvement beyond $m=25$ (with the expected delay going from $\tau = 13$ to $\tau = 11$ for $m=300$).

\subsection{Real data experiments}\label{ssec:real_data}

\begin{figure}[t!]
    \centering
    \includegraphics[width=0.9\linewidth]{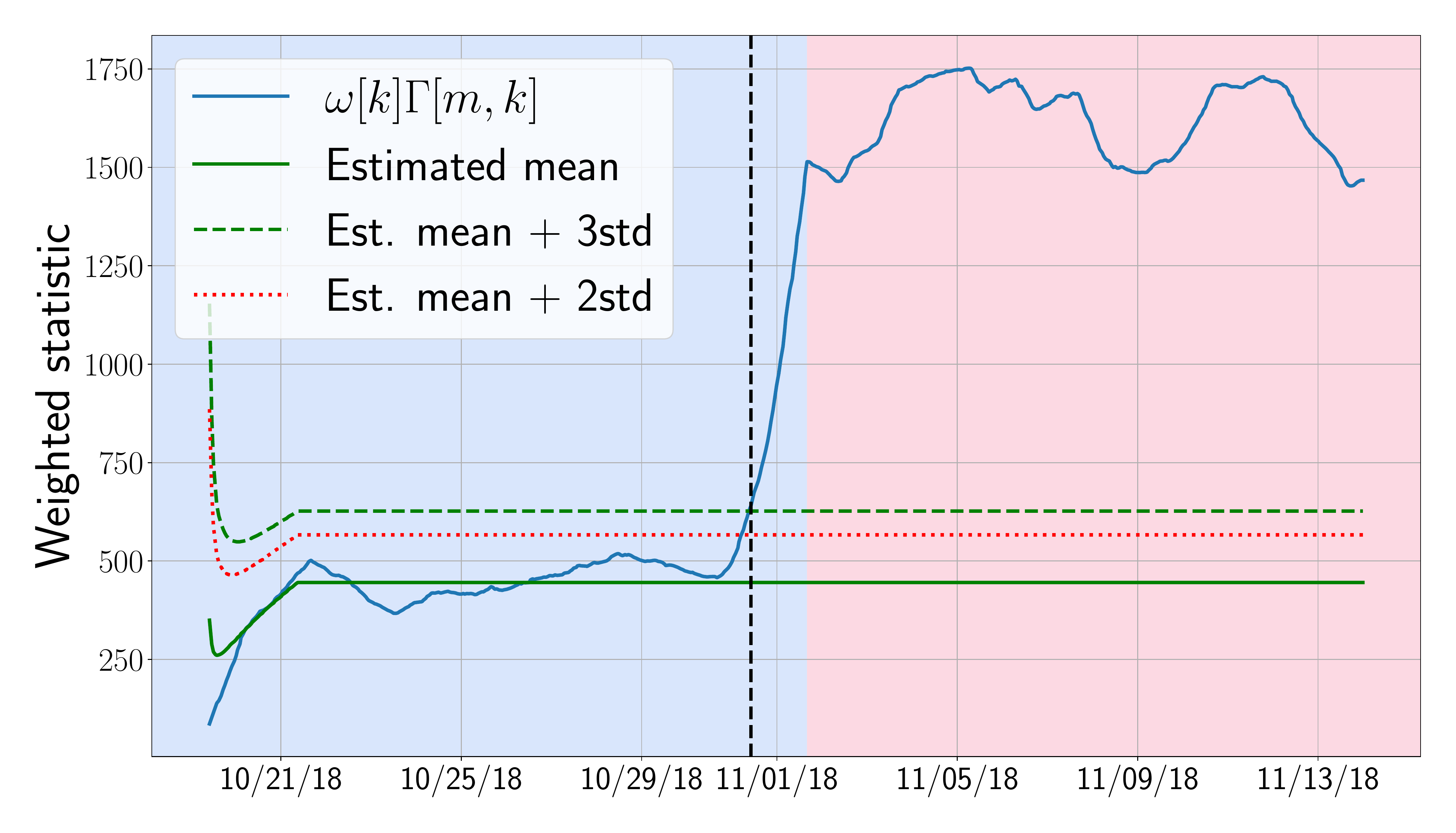}
    \includegraphics[width=0.9\linewidth]{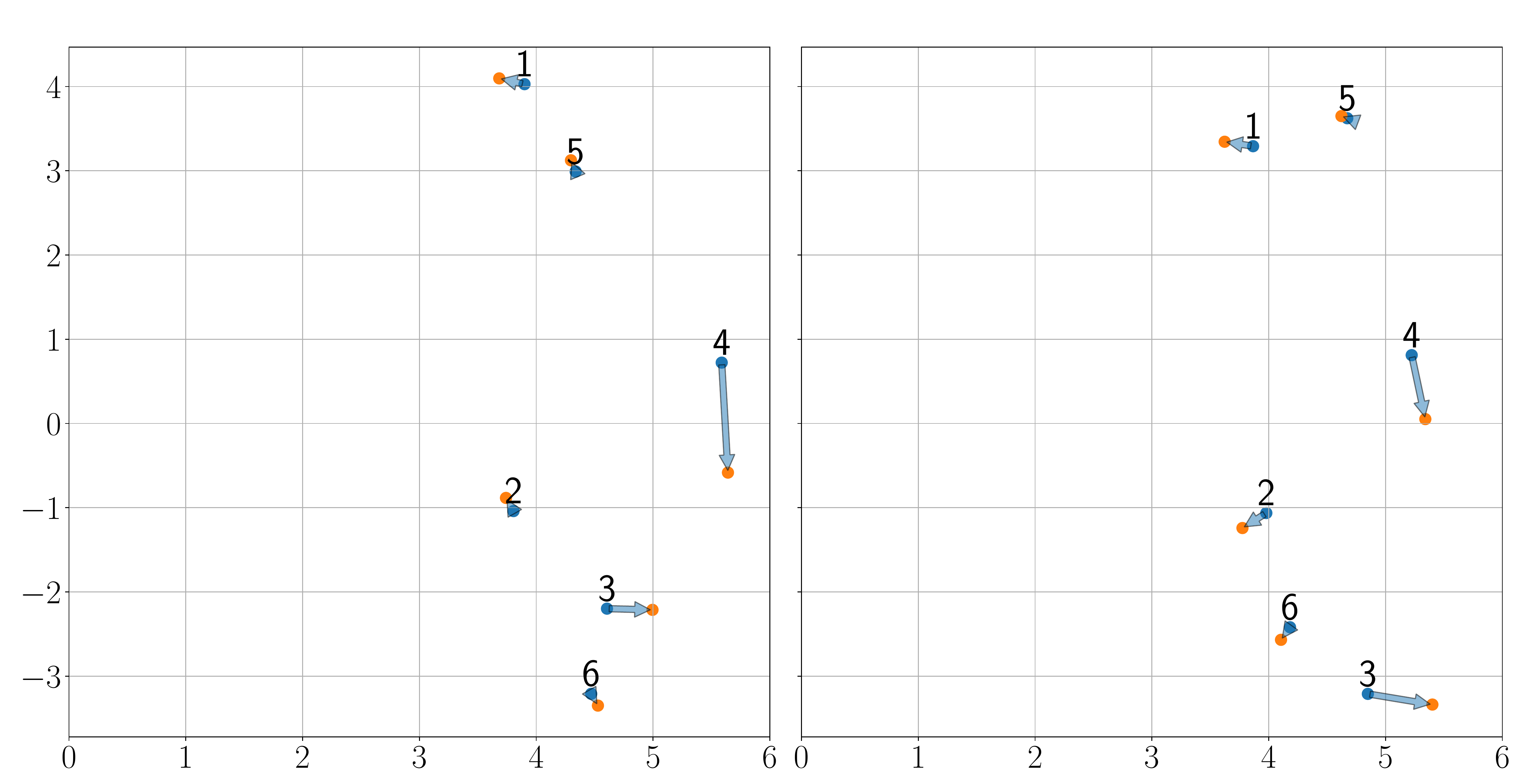}
    \caption{Online CPD for the RSSI dataset. Top: MOSUM statistic. A change in background color indicates a change-point detected by the offline algorithm~\cite{padilla2019change}. The dashed vertical line shows the detected change-point for the online algorithm. Algorithm \ref{A:online_CPD} successfully detects that an AP was moved. Bottom, left: $\hbX_1^l$ (blue) and $\hbX_2^l$ (orange) latent vectors for $d=2$ corresponding to $\bar{\bbA}_1$ and $\bar{\bbA}_2$ respectively. Vectors corresponding to the same node are joined by an arrow. Bottom, right: Id. but with $\hbX_1^r$ (blue) and $\hbX_2^r$ (orange).  Node 4 corresponds to the AP that was moved, which together with node 3 are the ones whose embeddings change more prominently.}
    \label{fig:rssi}
\end{figure}

\noindent \textbf{Wireless network data.} Received Signal Strength Indicator (RSSI) measurements between Wi-Fi access points (APs) in a Uruguayan school are obtained from the dataset described in~\cite{capdehourat2020nation}. In this particular example we considered a network consisting of $N=6$ APs, with  measurements collected hourly during almost four weeks, spanning from 10/17/2018 to 11/13/2018 (corresponding to $T=655$ graphs). The AP corresponding to node 4 was moved on 10/30/2018. 
As RSSI is measured in dBm (and are negative), we have first added an offset of $91$ to all weights so that they become positive (as $-90$\,dBm is the smallest RSSI measurement in this case) and that larger values still mean ``stronger'' edges. We thus have a directed (as power measurements between APs are not necessarily symmetric) and weighted graph sequence. 

We used two days worth of measurements for training ($m=48$) beginning on 10/12/2018. The resulting MOSUM statistic, the estimated mean and the resulting threshold $\text{th}[k]$ are shown in Figure~\ref{fig:rssi} (top). Note how Algorithm \ref{A:online_CPD} rapidly detects the AP movement. The offline CPD baseline in~\cite{padilla2019change} is also able to detect the change, at around the same date. Furthermore, we complement the threshold studies carried out in Section \ref{ssec:synthetics} and compare the same two versions of $\textrm{th}[k]$, namely the estimated mean plus two or three standard deviations. Note how the change-point is detected around the same instant regardless of the specific choice.

In addition to CPD, a valuable feature of RDPGs and its variants is their interpretability. To illustrate this attribute, let us consider two averaged adjacency matrices: those corresponding to the historic dataset and the last two days of the observation period. Let us denote the resulting matrices as $\bar{\bbA}_1$ and $\bar{\bbA}_2$, respectively, and analyze the resulting latent positions. In order to avoid the rotation ambiguities, we have used the so-called omnibus embedding~\cite{levin2017omnibus}, which in this case amounts to performing ASE to $\bbM = \bigl( \begin{smallmatrix}    \bar{\bbA}_1& (\bar{\bbA}_1+\bar{\bbA}_2)/2\\
    (\bar{\bbA}_1+\bar{\bbA}_2)/2 & \bar{\bbA}_2\end{smallmatrix}\bigr)$. 
This approach is only practical when jointly embedding a few adjacency matrices (two here), as the size of $\bbM$ increases rapidly with the number of matrices considered. 

Nodal vectors ($d=2$) are depicted in Figure~\ref{fig:rssi} (bottom), where an arrow shows the changes between the embeddings of $\bar{\bbA}_1$ and $\bar{\bbA}_2$. Notice how the largest changes correspond to nodes 3 and 4. The scaling ambiguity we discussed in Section \ref{subsec:weighted_rdpg} obscures which of the two APs was actually moved. Still, this monitoring tool would be valuable to network administrations as it identifies changes in a timely fashion and it provides a curated list of potentially problematic APs.\vspace{2pt} 


\begin{figure}[t]
    \centering
    \includegraphics[width=0.9\linewidth]{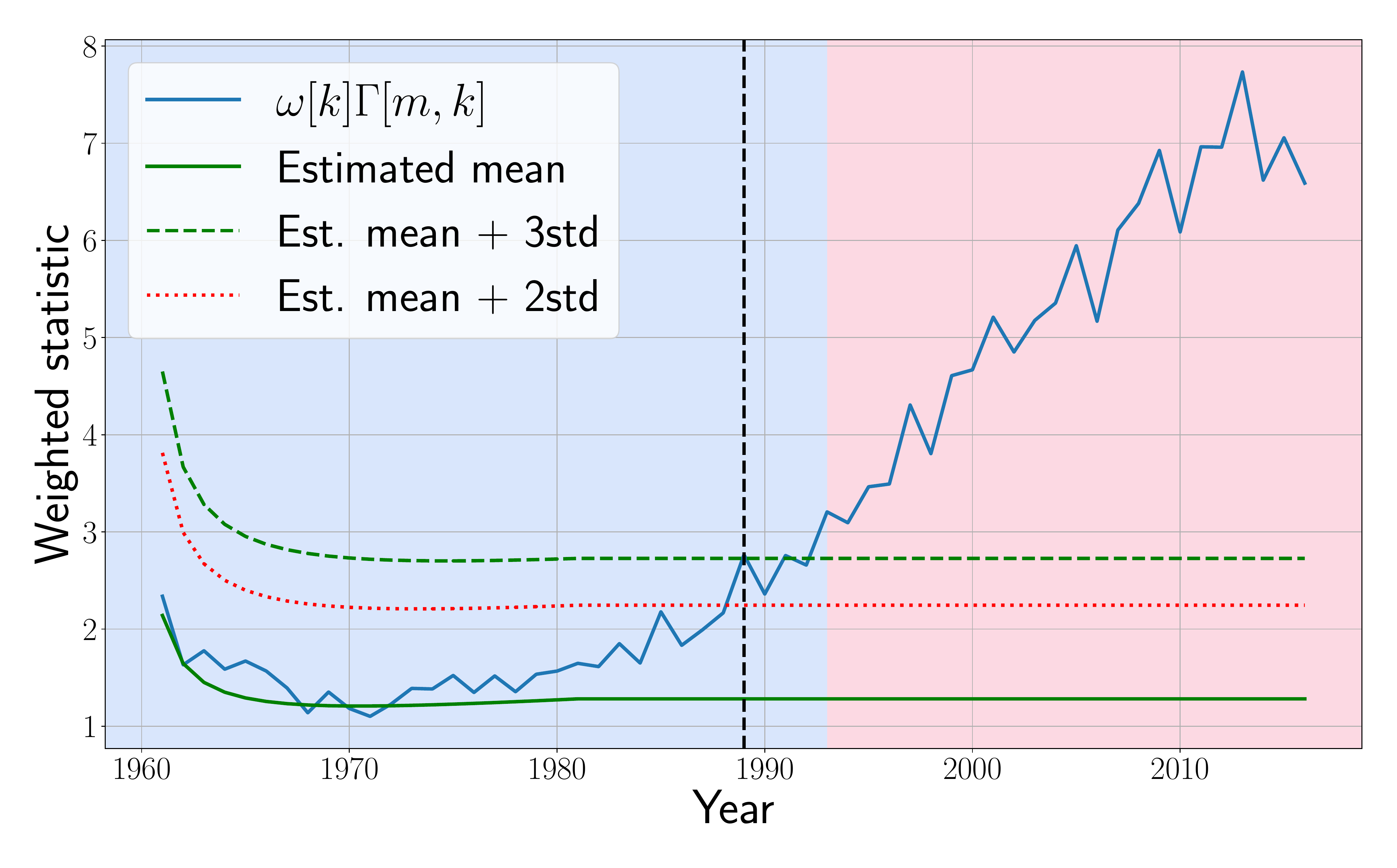}
    \includegraphics[width=0.9\linewidth]{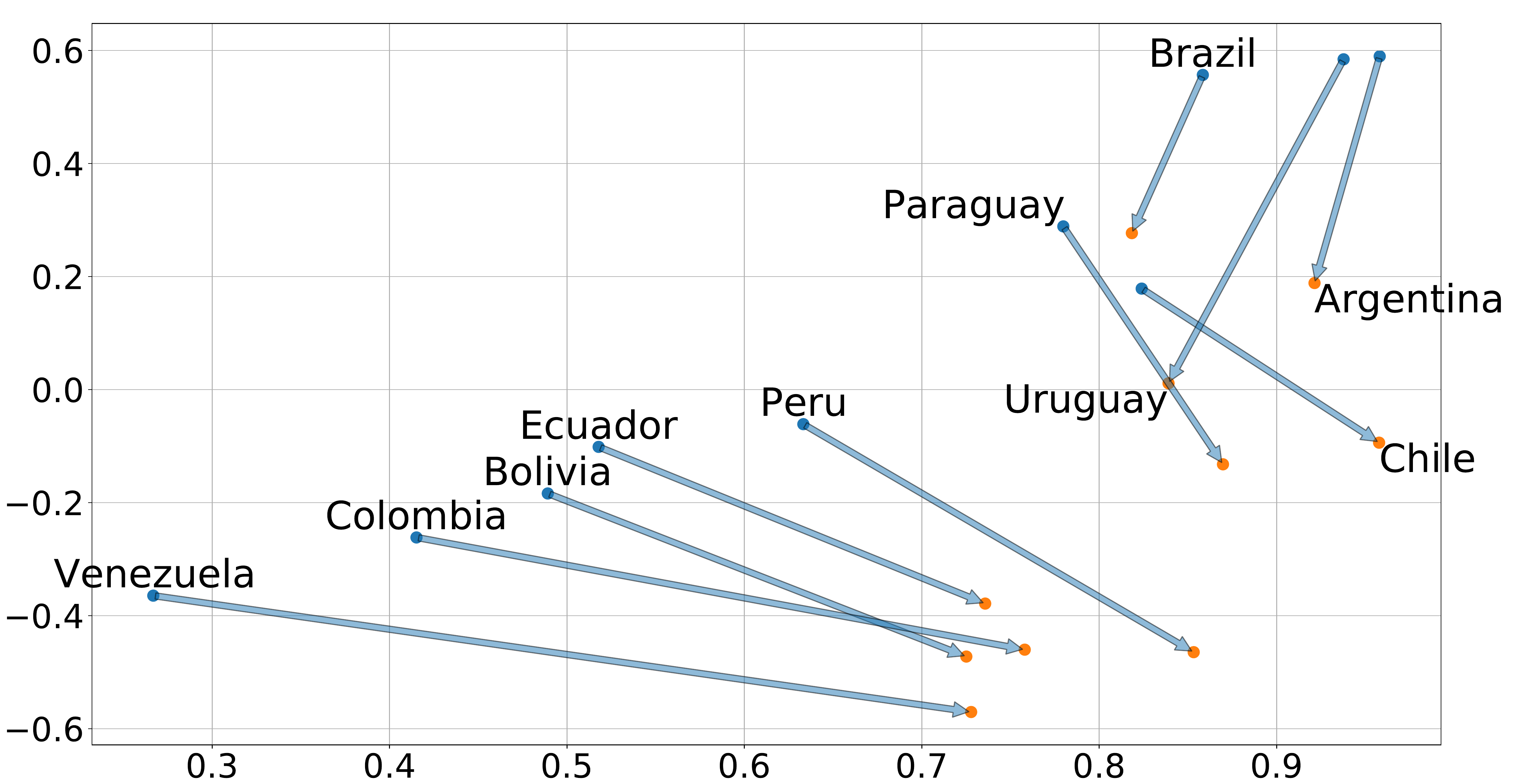}
    \caption{Online CPD for the South American football matches. Top: evolution of MOSUM statistic. The dashed vertical line shows the detected change-point, that can be traced to a change in the \emph{Copa Am\'erica} organization format. A change in background color indicates a change-point detected by the offline algorithm~\cite{padilla2019change}. Bottom: embeddings corresponding to the averaged historic set (blue) and the last 10 graphs of the observation period (orange). There are two distinct communities (northern and southern countries), and an increase of the number of matches played by the northern countries (with relatively less football tradition at the time) is clear by the changes in its embeddings.}
    \label{fig:conmebol}
\end{figure}

\noindent \textbf{South American football matches.} Consider a dynamic football network, whose $N=10$ nodes are the national teams affiliated to CONMEBOL (which associates all South American countries except Guyana and Suriname). This is the oldest continental confederation under FIFA, and its teams have a long history going back to 1901. We consider yearly matches since 1940, when all national associations were founded and most have joined CONMEBOL (Venezuela joined in 1952). 

The resulting undirected graphs have edge weights indicating the number of matches played between the two incident national teams during a particular year (data obtained from \url{https://www.eloratings.net/}). We used the first $m=20$ years for training and the evolution of the resulting weighted CUSUM statistic is shown in Figure \ref{fig:conmebol} (top). 

A change-point is detected around 1990 both by the online and offline CPD algorithms. Indeed, CONMEBOL's flagship tournament (\emph{Copa Am\'erica}) went through a period of intermittency that would last until 1987, when it started being organized regularly every two years with a nation hosting the event. This is apparent from the resulting embeddings in Figure \ref{fig:conmebol} (bottom), where northern countries increase their corresponding magnitudes (indicating more frequent matches) and form a relatively tight community. On the other hand, southern countries form another (more loose) community, which approached the northern's one in recent years. Furthermore, this community's structure changed, where e.g., the historic Argentina-Uruguay match is now not as significant. We also examine the robustness of the results with respect to the choice of the threshold. Notice that both versions of $\textrm{th}[h]$ we implemented again detect a change-point roughly around the same time (one year difference in Figure \ref{fig:conmebol}). But as mentioned in Section \ref{ssec:synthetics}, using the mean plus three standard deviations clearly provides more robustness to false positives, particularly in high noise settings as in this test case.

\begin{figure}[t]
    \centering
    \includegraphics[width=0.9\linewidth]{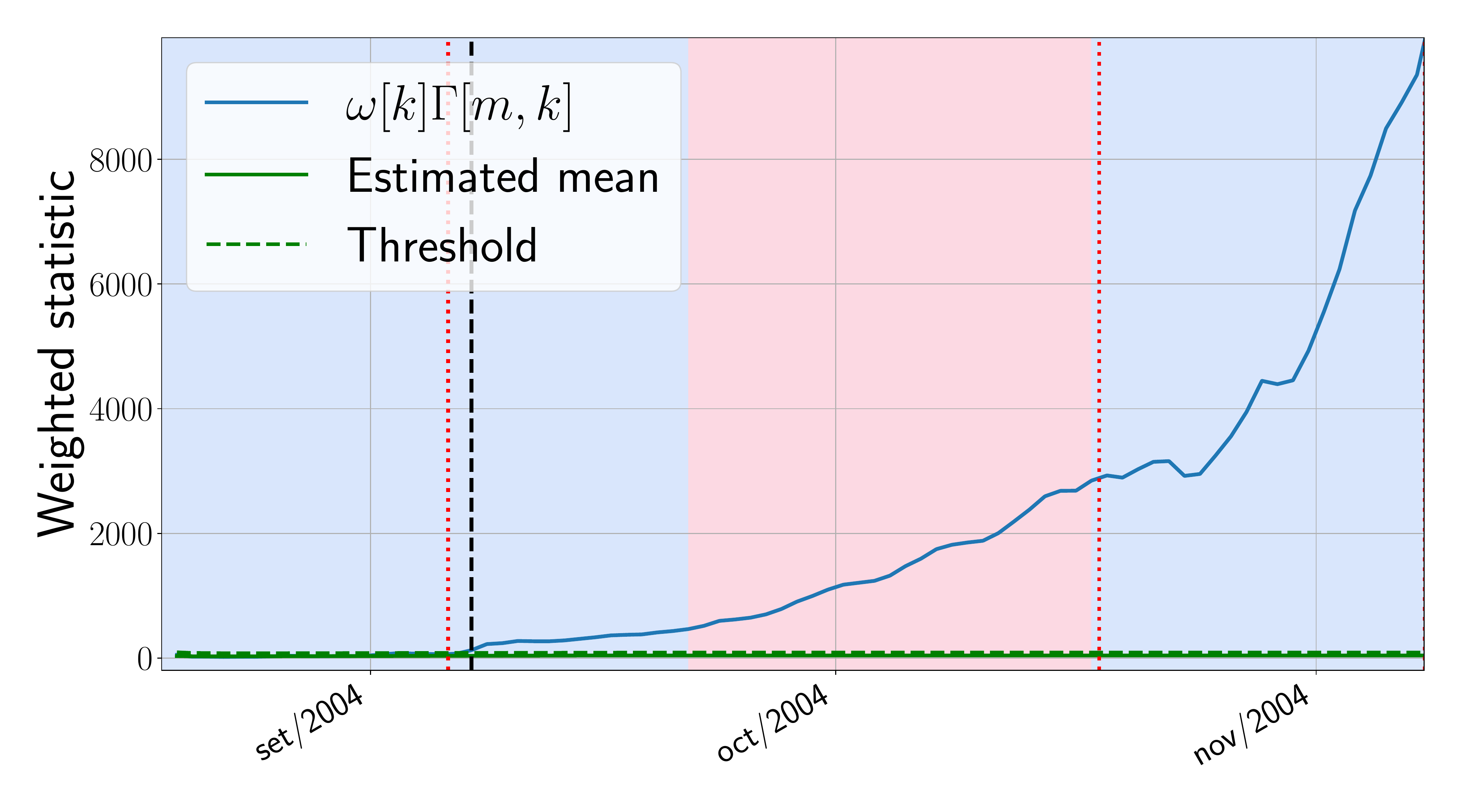}
    \caption{Online CPD for the MIT proximity dataset (using the MOSUM window). A change in background color indicates a change-point detected by the offline algorithm of~\cite{padilla2019change}. The dashed vertical line shows the detected change-point for the online algorithm. Dotted vertical lines indicate the beginning of the semester and the ``sponsor week''. The offline algorithm misses the first change-point. }
    \label{fig:mit}
\end{figure}
\vspace{2pt} 

\noindent \textbf{MIT proximity network.} Lastly, let us consider the stream of social graphs introduced in~\cite{eagle2006reality}. The dataset  includes the cell tower to which the mobile phone of a group of MIT faculty and graduate students connected between July 2004 and June 2005. We have processed the dataset and constructed a daily graph where nodes are people and the weight of each edge is how many minutes two people share the same tower on that given day\footnote{We used Jeremy Kun's scripts in \url{https://github.com/j2kun/reality-mining}.}. A collection of labeled events are described in \cite[Fig.~8]{peel2014evolvingCPD}, such as the beginning of the semester in early September and the ``sponsor week'' during mid-October.

We have considered a full month worth of undirected graphs starting on mid-July as training set and all the $N=84$ people that were registered during the study. The evolution of the MOSUM statistic until early November is shown in Figure~\ref{fig:mit}. Dotted vertical red lines indicate the two events we mentioned before, which fall within the observation period. First of all, it is important to note that the online CPD algorithm detects a change during early September, very near to the beginning of the semester. This change-point is missed by the offline algorithm in~\cite{padilla2019change} (see the changes on the background color), which indicates a change-point almost two weeks later. Furthermore, the example illustrates an interesting advantage of a finite-memory statistic such as MOSUM: the second change-point (this time correctly flagged by the offline algorithm) is also clearly discernible. Notice how the statistic is starting to stabilize around mid-October and then presents a large change of slope. Indeed, changes on the statistic after plateauing are indicative of further change-points.

\section{Conclusions and Future Work}\label{sec:conclusions}

We developed a computationally-efficient online CPD algorithm for monitoring applications involving streaming network data. The goal is to declare in (pseudo) real time when a sequence of observed graphs changes its underlying distribution. Leveraging the RDPG modeling framework and assuming historical ``clean'' data are available, the novel algorithm computes (offline) the ASE of the historical graphs (i.e., a training set) and then efficiently updates the cumulative sum of a monitoring function as data arrive sequentially-in-time. 
Statistical analysis of the monitored random sequence facilitates deriving meaningful detection thresholds to control type-I error rates, as well as to study the algorithm's detectability limits and to numerically predict delay behavior. Generalizations of the RDPG model to directed and weighted graphs markedly broaden the applicability of the novel online CPD framework, as illustrated through various real-data case studies.

This work opens up several exciting and challenging avenues for future work.
For instance, while still relying on RDPG modeling it would be of interest to explore sequential CPD formulations that minimize (or provide an explicit handle on) detection delay. Even in the present setting, carrying out a rigorous delay analysis would constitute a valuable contribution. {In all fairness, accomplishing this goal would be central towards fully solving the online change-point detection problem}. With regards to ASE-induced model estimation error, although in this work we presented a simple yet effective ``leave-one-out'' approach to approximate its value, a worthwhile future direction in our agenda is the study of theoretical bounds and guarantees for this plug-in statistic.
Our methodology detects changes in the model with respect to a training set of nominal graphs, and assumes that the number of nodes in the network does not change. Depending on the particular application, it may be interesting to consider the case where certain nodes are not always present on the network, and we are interested in only a subset of them.  Along these lines, we believe it would be worthwhile to develop embedding and CPD algorithms for partially observed graph streams, say due to sampling. Lastly, one could also envision online CPD schemes using just graph signal observations, because ASE-type embeddings are likely still computable from empirical signal covariance matrices under diffusion model assumptions.

\appendix

\subsection{Proof of the bound in Example \ref{Ex:ER_detectability}}\label{app:bound}

A sequence of ER graphs with connection probability $p$ changes to $q=p-\Delta$ at a certain time-step. Equation $\|\bbe+\bbdelta\|_2^2>\|\bbe\|_2^2$ in this case may be written as
\begin{equation}\label{eq:detect_cond_er}
    2\sum_{i=1}^N\sum_{j=i+1}^N E_{ij}>-\Delta\frac{N(N-1)}{2}, 
\end{equation}
where we have assumed that $\Delta>0$ (the analysis that follows is readily extended to $\Delta<0$). 
Recalling that in this case $\bbE=\hbx\hbx^\top-p\bbone_{N\times N}$ (with $\hbx\in\reals^{N\times 1}$), we rewrite \eqref{eq:detect_cond_er} as
\begin{gather}\label{eq:evento_deteccion}
    \hbx^\top(\bbone_{N\times N}-\bbI)\hbx>\left(p-\frac{\Delta}{2}\right)N(N-1).
\end{gather}
Since asymptotically (in $N$) $\hbx$ is a normal vector with mean $\bbmu=\sqrt{p}\bbone_{N\times 1}$ and covariance matrix $\bbSigma=\frac{(1-p)}{Nm}\bbI$~\cite{athreya2016limit, tang2018connectome,bourgade2017eigenvector}, we consider this asymptotic regime and use results about the statistics of quadratic forms of Gaussian vectors~\cite[Ch. 5]{rencher2008linear}. For instance, the resulting mean is 
\begin{align*}
    \mathbb{E}[\hbx^\top(\bbone_{N\times N}-\bbI)\hbx] ={}& \text{tr}\left[(\bbone_{N\times N}-\bbI)\bbSigma\right] + \bbmu^\top(\bbone_{N\times N}-\bbI)\bbmu\\
    ={}& pN(N-1).
\end{align*}

Comparing the equation above to \eqref{eq:evento_deteccion}, it follows we have to bound the probability that $\hbx^\top (\bbone_{N\times N}-\bbI)\bbx$ exceeds its mean minus $\Delta N(N-1)/2$. To this end we compute the variance of the quadratic form, which is (let $\sigma^2:=(1-p)/(Nm)$)
\begin{align*}
    \text{var}[\hbx^\top(\bbone_{N\times N}-\bbI)\hbx] ={}& 2\text{tr}\left[((\bbone_{N\times N}-\bbI)\bbSigma)^2 \right] \\
    & +4\bbmu^\top(\bbone_{N\times N}-\bbI)\bbSigma(\bbone_{N\times N}-\bbI)\bbmu\\
 ={}& 2\sigma^2N(N-1)(\sigma^2 + 2(N-1)p).
\end{align*}
Applying Chebyshev's inequality, the result follows.\hfill $\blacksquare$

\subsection{Proof sketch for Theorem \ref{theorem:wrdpg}} \label{ProofTheoremWRDPG}

We now give an overview of the necessary steps to prove Theorem \ref{theorem:wrdpg}. We adapt the arguments used in \cite{sussman2014consistent} to accommodate our setting; therefore, we will outline how their proof can be adapted to our case.

The following notation will be used throughout this section. We will denote the eigendecomposition of matrix $\bbB \in \reals^{N \times N}$ as $\bbV_B\bbLambda_B\bbV_B^\top$, with the elements in the diagonal of $\bbLambda_B$ in decreasing order.  $\hbLambda_B\in\reals^{d\times d}$ will denote the diagonal matrix with the $d$ largest eigenvalues of $\bbB$, and $\hbV_B\in\reals^{N\times d}$ will be the corresponding $d$ dominant eigenvectors. Recall we assume that $0 \leq B_{ij} < M$, for some $M>0$, and that $\{B_{ij}\}_{i<j}$ are independent with $\E{B_{ij}} = P_{ij}$, where $\bbP=\bbX \bbX^\top$ for some fixed $\bbX \in \reals^{N\times d}$. The eigendecomposition of $\bbP$ is $\bbV_P\bbLambda_P\bbV_P^\top$. Matrices $\hbLambda_P$ and $\hbV_P$ are similarly defined for $\bbP$. 

The first step of the proof is to show that, for large $N$, it almost surely holds that:
\begin{gather*}
    \|\bbB^2-\bbP^2\|_F < \sqrt{3M^4N^3\log N},
\end{gather*}
where $\bbB^2 = \bbB \times \bbB$ denotes the usual matrix product. The argument (in a more general setting) can be found in~\cite[Lemma 2]{fishkind2013consistent}. In our setting, the basic idea is to write
\begin{gather*}
    \bbB^2_{ij} - \bbP^2_{ij} = \sum_{k\neq ij} \left( \bbB_{ik}\bbB_{kj} - \bbP_{ik}\bbP_{kj}\right) - \bbP_{ii}\bbP_{ij} - \bbP_{ij}\bbP_{jj}.
\end{gather*}
Since $\bbB_{ik}\bbB_{kj}$ are independent for $k\neq i,j$ and $\bbP_{ij}$'s are bounded by $M$, we use Hoeffding's inequality to show that
\begin{gather*}
    \Pc{\left( \bbB^2_{ij} - \bbP^2_{ij} \right)^2 \geq 2M^4(N-2)\log N + M^4(4N-4)} \leq \frac{2}{N^4}.
\end{gather*}

Then, using the subadditivity property of  probability and the Borel-Cantelli Lemma, we can show that almost always
\begin{gather*}
    \sum_{i,j : i \neq j} \left( \bbB^2_{ij} - \bbP^2_{ij} \right)^2 \leq \frac{5}{2}M^4N^3 \log N.
\end{gather*}
Since $\left( \bbB^2_{ii} - \bbP^2_{ii} \right)^2 \leq M^4$, we finally conclude that
\begin{gather*}
    \|\bbB^2-\bbP^2\|^2_F \leq \frac{5}{2}M^4N^3 \log N + N M^4 < 3M^4N^3\log N
\end{gather*}
for sufficiently large $N$.

Once this is established, we apply a variant of the Davis-Kahan theorem to $\bbB^2$ and $\bbP^2$. Since the eigenvectors of $\bbB$ and $\bbB^2$ coincide (the same is true for $\bbP$ and $\bbP^2$) and we assume the eigengap for $\bbP$ is greater than $\delta N$ (and thus the eigengap for $\bbP^2$ is greater than $\delta^2 N^2$), \cite[Corollary 3]{yu2015useful} ensures that it is possible to choose the columns of $\bbV_{B}$ such that
\begin{gather*}
    \|(\bbV_{B})_{\cdot i} - (\bbV_{P})_{\cdot i}\|_2 \leq \frac{2^{3/2}}{\delta^2} \sqrt{3M^4\frac{\log N}{N}},
\end{gather*}
for every $i \leq d$, where $(\bbV_{B})_{\cdot i}$ denotes the $i$-th column of matrix $\bbV_{B}$. Since the first $d$ columns of $\bbV_{B}$ are the columns of $\hbV_{B}$ (the same is true for $\bbV_{P}$ and $\hbV_{P}$) this in turn implies that, for such a choice, 
\begin{gather*}
    \|\hbV_{B} - \hbV_{P}\|_F \leq C \sqrt{d} \sqrt{\frac{\log N}{N}},
\end{gather*}
where $C$ is a constant.

The rest of the proof follows, \textit{mutatis mutandis}, that of \cite{sussman2014consistent}. First, by writing
\begin{align*}
      \|\hbV_{B}\hbLambda_B^{1/2} - \hbV_{P}\hbLambda_P^{1/2}\|_F \leq{} &{}\|\hbV_{B}(\hbLambda_B^{1/2} - \hbLambda_P^{1/2})\|_F \\
      {}&{}+\|(\hbV_{B} - \hbV_{P})\hbLambda_P^{1/2}\|_F
\end{align*}
using the previous bounds we can show that
\begin{gather}\label{eq:bound_ase}
    \|\hbV_{B}\hbLambda_B^{1/2} - \hbV_{P}\hbLambda_P^{1/2}\|_F \leq C d \sqrt{\log N}.
\end{gather}

Because $\rank (\bbP) = d$, by defining  $\bbY:=\hbV_{P}\hbLambda_P^{1/2}$ we have that $\bbY\bbY^\top=\bbP= \bbX\bbX^\top$, so $\bbX = \bbY \bbW$ for some orthogonal $\bbW$. Thus, the bound in \eqref{eq:bound_ase} also holds for \mbox{$\|\hbV_{B}\hbLambda_B^{1/2}\bbW - \bbX\|_F$}. Since the ASE estimation of the latent positions is \mbox{$\hbX = \hbV_{B}\hbLambda_B$}, this implies that almost surely
\begin{gather*}
    \|\hbX\bbW - \bbX\|_F \leq C d \sqrt{\log N}.
\end{gather*}
The proof then concludes as in~\cite{sussman2014consistent}.

\bibliographystyle{IEEEtran}
%
\bibliography{citations}

\end{document}